\documentclass[sigconf]{acmart}

\usepackage{graphicx} % DO NOT CHANGE THIS
\usepackage{natbib}  % DO NOT CHANGE THIS AND DO NOT ADD ANY OPTIONS TO IT
\usepackage{caption} % DO NOT CHANGE THIS AND DO NOT ADD ANY OPTIONS TO IT
\usepackage[switch]{lineno}

\AtBeginDocument{%
  \providecommand\BibTeX{{%
    \normalfont B\kern-0.5em{\scshape i\kern-0.25em b}\kern-0.8em\TeX}}}

\usepackage[ruled, vlined, noend]{algorithm2e}
\usepackage{amsfonts}  % for mathbb command

\usepackage{subcaption}  % subfigure side by side
\usepackage{makecell}
\usepackage{amsmath}
\usepackage{graphicx,lipsum,wrapfig}% http://ctan.org/pkg/{graphicx,lipsum,wrapfig}

\usepackage{multirow} % multirow table
\usepackage{array}

\usepackage{color}

\usepackage[switch]{lineno}
\newtheorem{definition}{Definition}

\DeclareMathOperator*{\argmax}{arg\,max}
\DeclareMathOperator*{\argmin}{arg\,min}

\newenvironment{packeditemize}{
\begin{list}{$\bullet$}{
\setlength{\labelwidth}{8pt}
\setlength{\itemsep}{0pt}
\setlength{\leftmargin}{\labelwidth}
\addtolength{\leftmargin}{\labelsep}
\setlength{\parindent}{0pt}
\setlength{\listparindent}{\parindent}
\setlength{\parsep}{0pt}
\setlength{\topsep}{3pt}}}{\end{list}}

\begin{document}

\title{Stealing Deep Reinforcement Learning Models \\ for Fun and Profit}

% \author{\IEEEauthorblockN{Kangjie Chen, Shangwei Guo, Tianwei Zhang, Xiaofei Xie and Yang Liu}\\
% \IEEEauthorblockA{Nanyang Technological University, Singapore\\
% Email: \{kangjie.chen, shangwei.guo, tianwei.zhang, xfxie, yangliu\}@ntu.edu.sg}
% }

\author{Kangjie Chen}
\affiliation{%
  \institution{Nanyang Technological University}
  \country{Singapore}}
\email{kangjie.chen@ntu.edu.sg}

\author{Shangwei Guo}
\affiliation{%
  \institution{Nanyang Technological University}
  \country{Singapore}}
\email{shangwei.guo@ntu.edu.sg}

\author{Tianwei Zhang}
\affiliation{%
  \institution{Nanyang Technological University}
  \country{Singapore}}
\email{tianwei.zhang@ntu.edu.sg}

\author{Xiaofei Xie}
\affiliation{%
  \institution{Nanyang Technological University}
  \country{Singapore}}
\email{xfxie@ntu.edu.sg}

\author{Yang Liu}
\affiliation{%
  \institution{Nanyang Technological University}
  \country{Singapore}}
\email{yangliu@ntu.edu.sg}

\begin{abstract}
This paper presents the \emph{first} model extraction attack against Deep Reinforcement Learning (DRL), which enables an external adversary to precisely recover a black-box DRL model only from its interaction with the environment. Model extraction attacks against supervised Deep Learning models have been widely studied. However, those techniques cannot be applied to the reinforcement learning scenario due to DRL models' high complexity, stochasticity and limited observable information. We propose a novel methodology to overcome the above challenges. The key insight of our approach is that the process of DRL model extraction is equivalent to \emph{imitation learning}, a well-established solution to learn sequential decision-making policies. Based on this observation, our methodology first builds a classifier to reveal the training algorithm family of the targeted black-box DRL model only based on its predicted actions, and then leverages state-of-the-art imitation learning techniques to replicate the model from the identified algorithm family. Experimental results indicate that our methodology can effectively recover the DRL models with high fidelity and accuracy. We also demonstrate two use cases to show that our model extraction attack can (1) significantly improve the success rate of adversarial attacks, and (2) steal DRL models stealthily even they are protected by DNN watermarks. These pose a severe threat to the intellectual property and privacy protection of DRL applications. 
\end{abstract}

\iffalse


\fi

\pagenumbering{arabic}

\maketitle

\section{Introduction}
Deep Reinforcement Learning (DRL) has gained popularity due to its strong capability of handling complex tasks and environments. It integrates Deep Learning (DL) architectures and reinforcement learning algorithms to build sophisticated policies, which can accurately understand the environmental context (\emph{states}) and make the optimal decisions (\emph{actions}). Various algorithms and methodologies have been designed to facilitate the applications of DRL in different artificial intelligent tasks, e.g., autonomous driving \cite{lillicrap2015continuous}, robot motion planning \cite{zhu2017target}, video game playing \cite{mnih2015human}, etc.

As DRL has been widely commercialized (e.g., autonomous driving Wayve \cite{Wayve}, path planning MobilEye \cite{MobilEye}), it is important for model owners to protect the intellectual property (IP) of their DRL-based products. DRL models are generally deployed as black boxes inside the applications, so the model designs and parameters are not disclosed to the public, which prevents direct plagiarization or modification by malicious users. From the adversarial perspective, we want to investigate the following question in this paper: \emph{is it possible to extract the proprietary DRL model with only oracle access?} This is known as model extraction attack, which has been widely studied for supervised DL models \cite{oh2019towards, tramer2016stealing, duddu2018stealing}. However, the possibility and feasibility of extracting DRL models have not been explored yet. We make the first step towards this goal. 

% %\textbf{Problem statement.}
% From the adversarial perspective, we want to address the following question in this paper: \emph{is it possible to extract the properties (i.e., training algorithm family) of a black-box DRL model, and produce a replicated model with the same behaviors?} This is known as model extraction attack, which has been widely studied in supervised DL models \cite{oh2019towards, tramer2016stealing, duddu2018stealing}. However, the possibility and feasibility of extracting DRL models have not been explored yet. We make a first step towards this goal.

%\textbf{Challenges.}
Although various extraction techniques were designed against supervised DL models, challenges arise when applying them to DRL models due to significant differences of model features and scenarios. 
First, some attack approaches can only extract very simple models (e.g., two-layer neural networks \cite{jagielski2020high}) and datasets (e.g., MNIST \cite{oh2019towards}). 
% the method in \cite{jagielski2020high} can only work for two-layer neural networks. The method in \cite{oh2019towards} is only applicable to simple models with simple datasets (e.g., MNIST) constrained by the computing power. 
In contrast, DRL models usually have more complicated and deeper network structures to handle complex tasks, which cannot be extracted by the above techniques. 
Second, the adversary in the DRL environments has less observable information for model extraction. Past works assume the adversary has access to the prediction confidence scores \cite{jagielski2020high,oh2019towards,tramer2016stealing}, gradients \cite{milli2019model} or the side-channel execution characteristics \cite{hu2020deepsniffer,batina2019csi}. In our black-box DRL scenario, the adversary can only observe the predicted actions from the DRL model, which are not enough to recover the DRL model with the above methods. 
Third, supervised DL models perform predictions over discrete input samples, which are independent of each other. However, DRL is designed to solve Markov Decision Process (MDP) problems. Individual input samples cannot fully reflect the inherent features and behaviors of DRL models. The adversary will lose the information of temporal relationships if he only observes these discrete data. Besides, compared to supervised DL models, DRL models are more stochastic and their behaviors highly depend in the environments with different transition probabilities.

In this paper, we propose a novel model extraction methodology for DRL models, which can overcome the above challenges. The key insight of our methodology is that \emph{the process of DRL model extraction is equivalent to an imitation learning task}. Imitation learning \cite{hussein2017imitation} is a promising technique to learn and mimic the behaviors of an expert instead of training a policy directly based on a reward function. This exactly matches the scenario of DRL model extraction, which is to mimic the behaviors of the targeted model. We formalize the DRL model extraction attack, and prove its equivalence with imitation learning. 

Although there exists such a close connection, it is not easy to directly apply imitation learning for DRL model extraction. Given a task, DRL policies obtained from different training algorithms have significant differences in terms of behaviors and performance (as discussed in Section \ref{sec:bg-drl}). In imitation learning, a DRL algorithm needs to be specified as an oracle. If the adversary does not know the training algorithm employed by the targeted model (which is assumed in our threat model), he will have to randomly pick one algorithm as the basis of imitation learning. The imitated model will behave quite differently if the adversary selects a different algorithm from the victim model. 

Our solution addresses this issue by identifying the algorithm family of the DRL model first. We build an algorithm classifier, which is able to predict the algorithm family of a black-box DRL model based on its behaviors and environmental states. Specifically, (1) we use timing sequences of actions as the feature of a DRL model to characterize its decision process and interaction with the environment. (2) We utilize Recurrent Neural Networks as the structure of the classifier for training and prediction, which can better understand the temporal relationships inside the feature sequence. (3) For each DRL model, we generate different feature sequences in environments initialized with different random seeds. This guarantees that the training set of the classifier is comprehensive and includes different behaviors of the same algorithm family.

With the identified training algorithm, we are now able to use state-of-the-art techniques from the imitation learning community to conduct DRL model extraction attacks. Specifically, we adopt Generative Adversarial Imitation Learning (GAIL) \cite{ho2016generative}, which trains a discriminative model and generative model to imitate the behaviors of the targeted DRL policy. The contest between these two models can guarantee that our replication has very similar behaviors as the targeted one on the same environment. 

Our attack solution can extract models with high similarity of training algorithm families, behaviors and performance as the targeted models. Extensive experiments show the adversary can achieve 100\% of performance accuracy, and 97\% of behavior fidelity for various tasks and algorithms. We further provide two use cases to demonstrate that this attack can (1) significantly enhance the adversarial attacks by increasing the transferability of adversarial examples, and (2) easily invalidate the IP protection of watermarking mechanisms. This can bring severe economic loss as well as safety threats to the DRL-based applications. We expect this study can raise people's awareness about the privacy threats of DRL models, as well as the necessity of defense solutions.

The key contributions of this paper are:
\begin{packeditemize}
    \item We formally define DRL model extraction and prove the equivalence between this attack and imitation learning.
    \item We propose a novel method to identify the algorithm family of a black-box DRL model.
    \item We propose an end-to-end DRL model extraction attack based on imitation learning.
    \item We conduct extensive experiments and case studies to illustrate the effectiveness and severity of our attack method.
\end{packeditemize}

% organization
The rest of this paper is organized as follows. Section \ref{sec:background} introduces the background of DRL. Section \ref{sec:attack} presents the threat model and formal definition of DRL model extraction. We also give formal analysis of its equivalence with imitation learning. Section \ref{sec:algorithm} exhibits our DRL extraction attack. Section \ref{sec:exp} evaluates the effectiveness of the proposed attack, followed by two case studies in Section \ref{sec:adv-exp}. Section \ref{sec:discussion} discusses some open issues. We summarize related works in Section \ref{sec:work} and conclude in Section \ref{sec:conclusion}.

\section{Background}\label{sec:background}
\subsection{Reinforcement Learning}
Reinforcement Learning (RL) is a type of machine learning technology that enables an agent to interact with an environment and learn an optimal policy by maximizing the cumulative reward from the environment.
A RL problem can be modeled as a Markov Decision Process (MDP), represented as a tuple $(\mathbb{S}, \mathbb{A}, \mathbb{P}, r, \gamma)$, where 
\begin{packeditemize}
    \item $\mathbb{S}$ is a finite state space, which contains all the valid states in the environment;
    \item $\mathbb{A}$ is a finite action space, from which the agent chooses an action as the response to the state it observes;
    \item $\mathbb{P}: \mathbb{S} \times \mathbb{A} \times \mathbb{S} \rightarrow [0,1]$ is the state transition probability. For two states $s, s'$ and an action $a$, the output of $\mathbb{P}$ denotes the probability that $s$ is transited to $s'$ by taking action $a$;
    \item $r(s,a)$ is the reward function that outputs the expected reward if the agent takes action $a$ at state $s$;
    \item $\gamma \in [0,1)$ is the discount factor that denotes how much the agent cares about rewards in the distant future relative to those in the immediate future. A smaller factor values places more emphasis on the immediate rewards.
\end{packeditemize}

A RL policy $\pi: \mathbb{S} \times \mathbb{A} \rightarrow [0, 1]$ describes the behaviors of an agent in an MDP. It denotes the probability of an action $a \in \mathbb{A}$ the agent will take on a state $s\in \mathbb{S}$. Then the goal of a reinforcement learning is to identify a policy $\pi^*$ that maximizes the expected cumulative rewards:
    \begin{equation}
        \pi^* = \argmax_{\pi}\sum_{t=t_0}^{T} \sum_{a_t \in \mathbb{A}(s_t)} \gamma^{t-t_0} r(s_t, a_t) \pi(s_t, a_t)
    \end{equation}
where $T$ is the termination timestep. In practice, instead of observing $\pi$'s action distribution probability on a certain state $s$, we usually only capture $\pi$'s optimal action $a$, and thus the policy can be formulated as $\pi(s) = a$.

\subsection{RL with Deep Neural Networks}\label{sec:bg-drl}
Despite RL has been studied for a long time and achieved tremendous success in some tasks \cite{ng2006autonomous}, traditional approaches to solve the RL problem lack scalability and are inherently limited to relatively simple environments.
Deep Reinforcement Learning is then introduced, which adopts Deep Neural Networks (DNNs) to understand and interpret complex environmental states, and make the optimal decisions. Due to the great capabilities of neural networks in learning high-dimensional feature representations and function approximation properties, DRL can achieve outstanding performance in mastering human-level control policies in various tasks with high-dimensional states \cite{lillicrap2015continuous, zhu2017target}. 
% For instance, Convolutional Neural Networks (CNNs) are used as the core of some RL agents to learn directly from raw and high-dimensional visual inputs. 
There are generally three common approaches to solve reinforcement learning tasks.

%There are two main approaches to solving RL problems: methods based on value functions and methods based on policy search. There is also a hybrid, actor-critic approach, which employs both value functions and policy search. 

\noindent\textbf{Value-based Approach.}
% bellmen equation
The agent performs certain actions according to its policy to maximize its reward. The optimal behaviors of the policy $\pi$ are defined by the Q-function which obeys the following Bellman equation,
\begin{equation}
Q^\pi(s, a) = \mathbb{E}[r(s,a) + \gamma \max_{a'}Q^\pi(s', a')].
\end{equation}

This equation shows the maximum return value $Q^\pi(s, a)$ from state $s$ and action $a$ is the sum of the immediate reward $r$ and the return obtained following the optimal policy until the end of the episode.
When the agent interacts with the environment and transits from state $s$ to the next one $s'$, this approach estimates the value of $Q^\pi(s, a)$. Once we obtain all the values of each state-action pair, we can select the optimal action $a^*$ with the highest Q value on the current state $s$ (i.e., $a^* = \arg\max_{a}Q^\pi(s, a)$).

% limitation of Q-learning
For most problems, however, it is impractical to represent the Q-function as a table containing the values of all possible combinations of $s$ and $a$.
Deep Q-Learning (DQN) was introduced to approximate the Q-value for each action. This algorithm has been extensively used to play GO \cite{silver2016mastering} and Atari games (at superhuman level) \cite{mnih2015human}. However, DQN cannot be adopted in the tasks with continuous action space since the algorithm requires to learn all the possible Q values.

\noindent\textbf{Policy-based Approach.}
This solution attempts to identify the optimal policy directly other than estimating all the state-action values.
% At each state $s_t$, the policy $\pi_\theta$ samples an action $a_t$ following the action probability $\mathbb{P}(a|s, \theta)$. 
Typical examples include REINFORCE \cite{williams1992simple} which regards an RL policy as a function $\pi_\theta(s, a) = \mathbb{P}(a|s, \theta)$ and optimizes it by applying the policy gradient technique. 
To extend this approach to complex tasks, researchers modeled the policy $\pi_\theta$ with DNNs such as Multilayers Perceptron and Convolutional Neural Networks. The objective function of the policy network is defined as the expectation of the total discounted rewards on all the states of a trajectory in an episode,

\begin{equation}\label{equ: policy_based_obj_F}
    J(\theta) = \mathbb{E}_{(s_t, a_t)_{\verb|~|} \pi_\theta}[\sum_{t=0}^{\infty}\gamma^tr_t],
\end{equation}

% To maximize this objective function, gradient ascent is adopted to update the network parameters $\theta$ iteratively with the gradient $\nabla_\theta J(\theta)$ until the function reaches a local maximum. 
This approach has some limitations. Since the expected reward depends on all the states within the episode, if the agent receives a high reward, it tends to conclude all the actions taken on all the states were good, even if some of them were really bad. Moreover, as the network can only be updated after one episode is completed and the sample of one trajectory can be used for only once, data collection and sample utilization are inefficient in this approach.

\noindent\textbf{Actor-Critic Approach.}
This is an effective method to overcome the common drawbacks of policy-based methods.
% The third category is the hybrid of value function and policy search (actor-critic approach). 
It learns both a policy (actor) and a state value function (critic) to reduce variance and accelerate learning. An actor is formed with a policy network, similar as the policy-based approach, which performs action $a$ on the current state $s$. The critic is a Q-function represented by a network to estimate how good the action $a$ given by the actor at state $s$ is.  
As specified in \cite{mnih2016asynchronous}, the model can be learned with two objective functions: a) the objective function of the actor is the same as Equation \ref{equ: policy_based_obj_F}; b) the advantage function $A\pi_\theta(s, a)$ of the critic represents the extra reward the agent gets if it takes this action:
\begin{equation}
    A\pi_\theta(s, a) = Q_{\pi_\theta}(s, a) - (r + \gamma\max_{a'}Q_{\pi_\theta}(s', a')).
\end{equation}
Therefore, the actor-critic method combines the advantages of both policy-based and value-based methods. The actor enjoys the benefits of computing continuous actions without the need for optimization on a Q-function. The critic’s merit is that it supplies the actor with low-variance knowledge of the performance. These properties make the actor-critic methods an attractive reinforcement learning solution. State-of-the-art algorithms include Proximal Policy Optimisation (PPO) \cite{schulman2017proximal}, Actor-Critic with Experience Replay (ACER) \cite{wang2016sample} and Actor Critic using Kronecker-Factored Trust Region (ACKTR) \cite{wu2017scalable}.

\section{Attack Formulation and Analysis} \label{sec:attack}

\subsection{Threat Model}
We consider a standard scenario of model extraction, where the adversary has oracle access to the targeted model. Specifically, the adversary can operate the DRL model in the normal environment without manipulating the states. He can observe the states and the model's corresponding actions. This is common when the DRL application is deployed in the open and public environment. For instance, in DRL-based robotics, the adversary can deploy the robot in certain environments, and observe its corresponding reactions. In the application of video game playing, the adversary can start a DRL player with different situations, and collect its strategies. 

Different from prior works of extracting supervised DL models \cite{jagielski2020high,oh2019towards,tramer2016stealing,carlini2020cryptanalytic}, we consider more restricted capabilities for the adversary: he has \textit{no} knowledge about the targeted DRL model, including the model structures and parameter values, training algorithms and hyperparameters, etc. He is \textit{not} able to obtain the confidence score of each possible action predicted by the targeted DRL policy. This makes the attack more practical, yet more challenging as well. 

\subsection{Formal Definition of DRL Model Extraction}

According to \cite{jagielski2020high}, there are two basic categories of model extraction attacks. (1) \emph{Accuracy extraction} aims to reproduce a model which can match or exceed the accuracy of the targeted model. The replicated model does not need to match the predictions of the oracle precisely. (2) \emph{Fidelity extraction} attempts to recover a model with the same behaviors as the victim one, even for incorrect prediction results. In this paper, we focus on both accuracy and fidelity extraction of DRL models, which is more threatening. 

It is worth noting that we only consider to extract an approximate copy of the targeted model with similar rewards and behaviors, rather than the exact values of model parameters. Parameter extraction is difficult for complex DL models \cite{jagielski2020high}, as different values can give the same model outputs, which are indistinguishable from the adversary. Besides, the adversary cannot identify the values of dead neurons which never contribute to the model output. 

Let $\mathbb{F}$ be the hypothesis class of a DRL extraction attack. The instance and output spaces of $\mathbb{F}$ are denoted as $\mathbb{A} \times \mathbb{S}$ and $[0, 1]$, respectively. We refer to the sequence $S = \{s_1, s_2, ..., s_T\}$, $s_i \in \mathbb{S}$ as the environmental states of one episode, and $A = \{a_1, a_2, ..., a_T\}$, $a_i \in \mathbb{A}$ as the corresponding actions of the targeted model. The adversary’s goal is to find a policy $\pi \in \mathbb{F}$ such that $\pi$ performs competitively with the targeted policy $\pi^*$. 
Following the assumptions in the threat model where the adversary can operate the model in an environment and learn $\pi$ by observing the corresponding actions of $\pi^*$. We now formally define the DRL extraction attack:
% Following the assumptions in the threat model, we now formally define the DRL extraction attack, where the adversary can operate the model in an environment and learn $\pi$ by observing the corresponding actions of $\pi^*$.

\begin{definition}\label{def:attack}($(\delta, \xi)$-DRL Extraction Attack)
    Let $\pi$ be the reproduced policy in a DRL extraction attack $\mathcal{A}$. We call $\mathcal{A}$ a $(\delta, \xi)$-DRL extraction attack if for $\forall S = \{s_1, s_2, ..., s_T \}$, the difference between the expected cumulative rewards of $\pi^*$ and $\pi$ is smaller than $\delta$ (accuracy extraction) and the expected distance between the action distributions of $\pi^*$ and $\pi$ is smaller than $\xi$ (fidelity extraction), i.e.,
    \begin{align}\label{equ:reward}
        & |E[\sum_{i=1}^{T}r(s_i,a_i^*)] - E[\sum_{i=1}^{T}r(s_i,a_i)]| \leq \delta,\\\label{equ:behavior}
        & E[d(\{a_1^*, a_2^*, ..., a_T^* \}, \{a_1, a_2, ..., a_T\})] \leq \xi,
    \end{align}
    where $\{a_1^*, a_2^*, ..., a_T^* \}$ and $\{a_1, a_2, ..., a_T\}$ are the corresponding actions of $\pi^*$ and $\pi$ on the sequence of states. $E$ is the expected value of a random variable.
\end{definition}

\subsection{Imitation Learning}

Imitation learning  was originally introduced for learning from human demonstrations. Then its concept was applied to the domain of artificial experts, such as RL agents. Given a task, an imitation model acquires the corresponding skills from expert demonstrations by learning a mapping between observations and actions \cite{hussein2017imitation}. A formal definition of imitation learning is described below:

\begin{definition}(Imitation Learning System)
    Let $\mathbb{F}'$ be a hypothesis class with instance space $\mathbb{X}$ and output space $\mathbb{Y}$. An imitation learning system for $\mathbb{F}'$ is given by two entities ($\mathcal{A}$, $\mathcal{O}$): 
%    an agent (the learning machine) $\mathcal{A}$ and a teacher $\mathcal{O}$. 
    $\mathcal{O}$ is the teacher that plays as an oracle with the optimal hypothesis. $\mathcal{A}$ is the agent that searches or learns the optimal hypothesis from $\mathbb{F}'$ by observing $\mathcal{O}$'s behaviors $\{(x, y)\}$, $x\in \mathbb{X}, y \in \mathbb{Y}$. 
\end{definition} 

The agent $\mathcal{A}$ adopts an imitation learning algorithm $\mathcal{L}$ to learn the behaviors of the teacher $\mathcal{O}$ in the system. We define such an algorithm below:

\begin{definition}($\epsilon$-Imitation Learning Algorithm)
    Let $\mathcal{L}$ be an algorithm used by $\mathcal{A}$ to learn the behaviors of $\mathcal{O}$.
    Let $\pi$ be an ``expert'' that describes $\mathcal{O}$'s behaviors, and 
    %as one of the hypothesises in the hypothesis class of an imitation learning algorithm. 
    $\pi'$ be the hypothesis learned by $\mathcal{L}$. We say $\mathcal{L}$ is an $\epsilon$-imitation learning algorithm if for $\forall \{s_1, s_2, ..., s_T \}$, the expected value of the distance between the corresponding action distributions of $\pi^*$ and $\pi'$ is smaller than $\epsilon$, i.e.,
    \begin{equation}
        E[d(\{a_1^*, a_2^*, ..., a_T^* \}, \{a_1', a_2', ..., a_T'\})] \leq \epsilon,
    \end{equation}
    where $d(\cdot,\cdot)$ is a measure of two probability distributions.
\end{definition}
Intuitively, an algorithm is an $\epsilon$-imitation learning algorithm if the action distribution of the hypothesis learned by it is similar to the corresponding teacher with the distance bounded by the threshold $\epsilon$. In our experiments, we adopt the Jensen-Shannon divergence to measure the action distributions of two models. 

\subsection{Connections between DRL Model Extraction and Imitation Learning}
Based on the descriptions in the previous two sections, we can observe that DRL model extraction closely resembles the DRL imitation learning scenario. In this section, we formally prove that DRL model extraction can be cast as imitation learning under the assumption that the hypothesis classes of the DRL extraction attack and imitation learning are the same: 

\begin{theorem}
\label{the:equivalence}
    Let $\mathbb{F}^*$ be the hypothesis class for searching the targeted policy $\pi^*$ in an imitation learning system, $\mathbb{F}$ be the hypothesis class of a model extraction adversary with the same instance and output spaces as $\mathbb{F}^*$. We assume the adversary owns an $\epsilon_{\delta}$-imitation learning algorithm, and uses it to learn a policy $\pi'$. Let $t_{\epsilon}$ be the expected value of the number of different actions between $\pi^*$ and $\pi'$. Then this adversary is able to conduct a $(\delta, \xi)$-DRL extraction attack if (1) $\mathbb{F} = \mathbb{F}^*$, (2) the maximal reward of $\pi'$ on $t_{\epsilon}$ actions is smaller than $\delta$, and (3) $\xi \leq \epsilon_{\delta}$.
\end{theorem}

\begin{proof}
\begin{figure}[t]
    \centering
    \includegraphics[width=\columnwidth]{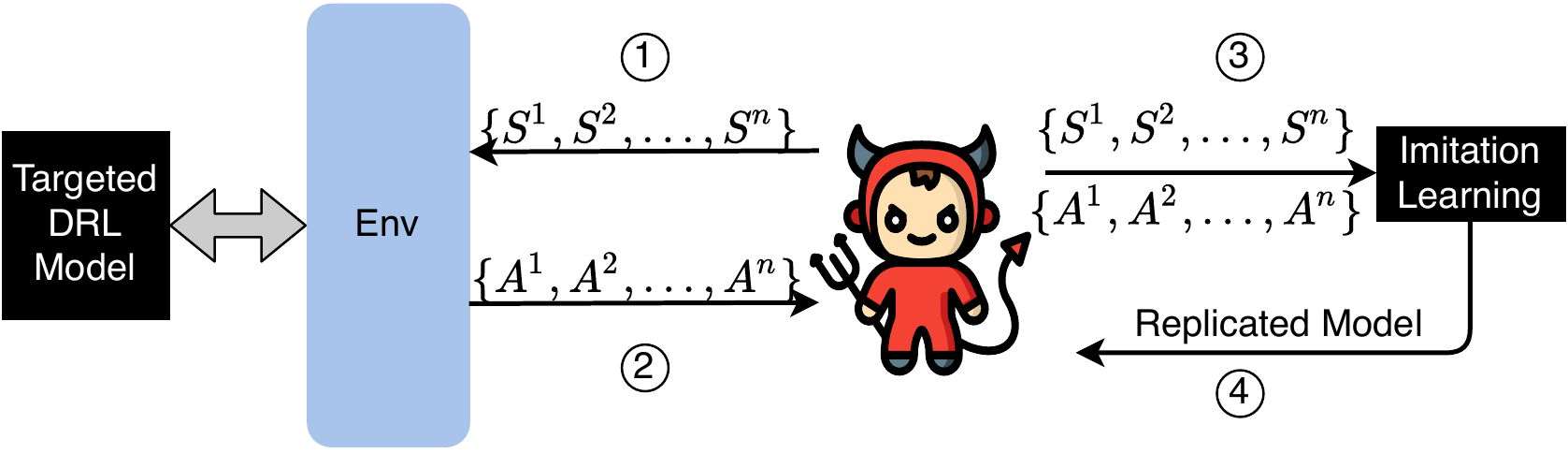}
    \caption{Constructing a $(\delta, \xi)$-DRL extraction attack based on an $\epsilon_{\delta}$-imitation learning algorithm.}
    \label{fig:proof}
\end{figure}

We construct a two-step DRL extraction attack, as shown in Figure \ref{fig:proof}. At step 1, the adversary randomly selects some states $\{S^1, S^2,..., S^n\}$, where $S^i = \{s_1^i,...,s_T^i\}$, and feeds them to the targeted DRL model (\raisebox{.5pt}{\textcircled{\raisebox{-.9pt} {1}}}). He then observes the corresponding actions from this model $\{A^1, A^2,..., A^n\}$,  where $A^i = \{a_1^i,...,a_T^i\}$ (\raisebox{.5pt}{\textcircled{\raisebox{-.9pt} {2}}}).
At step 2, with the collected states and actions, the adversary adopts an $\epsilon_{\delta}$-imitation learning algorithm to learn a new policy that can mimic the actions from the states (\raisebox{.5pt}{\textcircled{\raisebox{-.9pt} {3}}}). The adversary can repeat the two steps for multiple times until a qualified policy is acquired, which will be the output of the model extraction attack (\raisebox{.5pt}{\textcircled{\raisebox{-.9pt} {4}}}). 

%In the second step, the adversary takes the sequential imitation learning algorithm as an oracle and then sends $\{S^1, S^2,..., S^n\}$ and the corresponding actions to the oracle (Figure \ref{fig:proof} \raisebox{.5pt}{\textcircled{\raisebox{-.9pt} {3}}}). The adversary can repeat the two steps multiple times until the imitation learning algorithm is $\epsilon_{\delta}$-sequential imitation learning. Finally, the oracle outputs the learned model to the adversary.

We now prove that this two-step attack is a $(\delta, \xi)$-DRL extraction attack.
One can easily find that the attack satisfies Equation \ref{equ:behavior} because $\xi \leq \epsilon_{\delta}$. We show that the attack also satisfies Equation \ref{equ:reward} as follows.
Since an $\epsilon_{\delta}$-imitation learning algorithm is adopted, for $\forall \{s_1, s_2, ..., s_T \}$, we have
\begin{equation}
    E[d(\{a_1^*, a_2^*, ..., a_T^* \}, \{a_1', a_2', ..., a_T'\})] \leq \epsilon_{\delta}.
\end{equation}

$t_{\epsilon}$ is the expected number of actions differ in $\pi^*$ and $\pi'$. Here, we denote $\{(s_1^{max},a_1^{max}),...,(s_{t_{\epsilon}}^{max},a_{t_{\epsilon}}^{max}) \}$ as the state-action pairs that obtain the maximal reward than other $t_{\epsilon}$ pairs, i.e., 
$$\delta_{t_{\epsilon}} = \sum_{i=1}^{t_{\epsilon}}r(s_i^{max},a_i^{max}) \geq \sum_{i=1}^{t_{\epsilon}}r(s_i,a_i'), \ \text{for} \ \forall \ s_i \in \mathbb{S}.$$

Since $\delta_{t_{\epsilon}} < \delta$, we have
\begin{align}
    &|E[\sum_{i=1}^{T}r(s_i,a_i^*)] - E[\sum_{i=1}^{T}r(s_i,a_i)]|\\
    \leq&E[\sum_{i=1}^{t_{\epsilon}}r(s_i,a_i')] \leq \delta_{t_{\epsilon}} \leq \delta,
\end{align}
\end{proof}

% passively
% Instead of actively sending queries to the environment, the adversary can also passively observe the states and actions of the targeted model and achieve DRL extraction attack.

\section{Proposed Attack}\label{sec:algorithm}
\begin{figure*}[htp]
    \centering
    \includegraphics[width=0.85\textwidth]{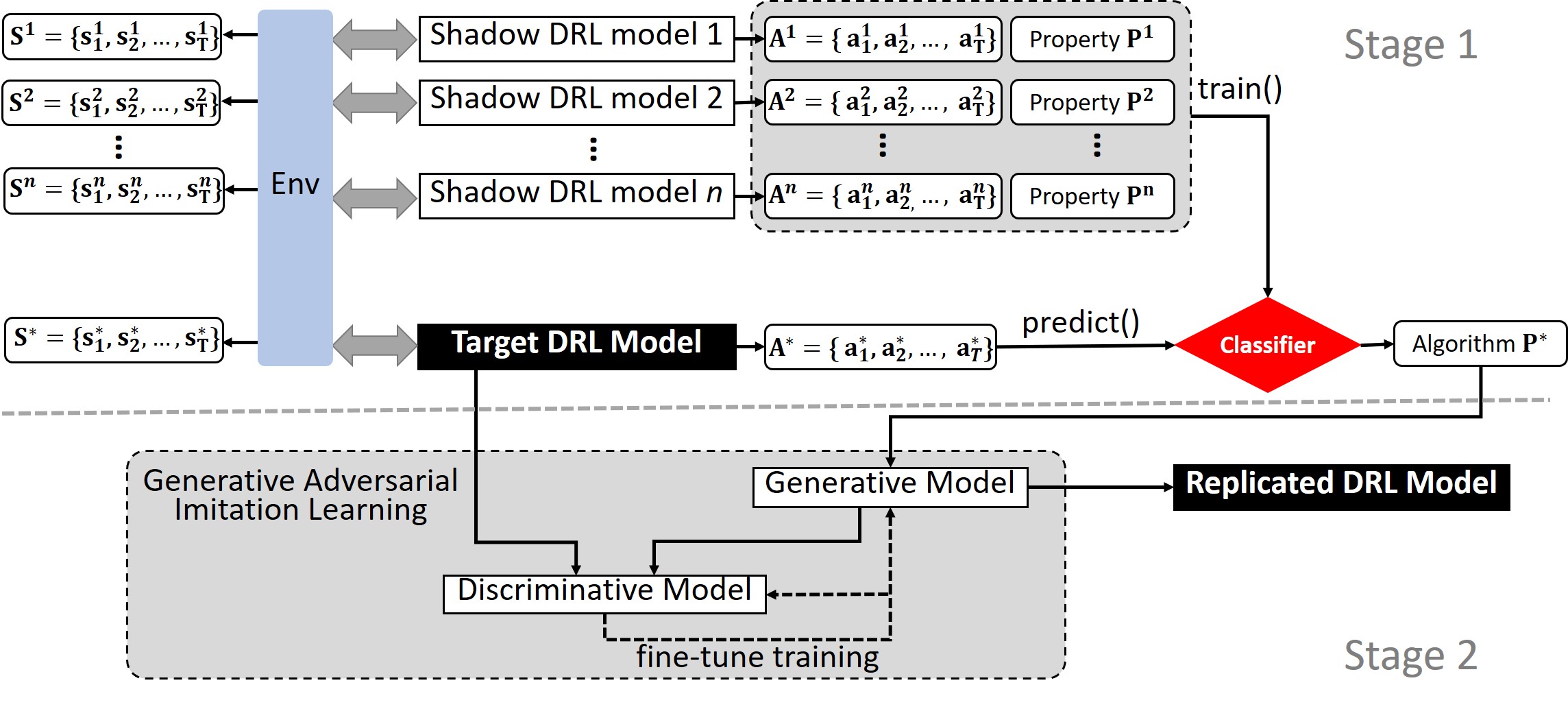}
    \caption{Overview of our proposed DRL model extraction attack}
    \label{fig:algorithm}
\end{figure*}

% design strategy
Based on Theorem \ref{the:equivalence}, we can adopt well-developed imitation learning techniques to extract DRL models. However, it also requires the hypothesis classes of the two problems must be the same. In another word, the adversary should use the same training algorithm of the targeted DRL model as the imitation learning algorithm. Unfortunately, we consider the adversary has no prior knowledge about the training algorithm of the DRL model. This makes it difficult to directly apply imitation learning solutions for model extraction. Due to the distinct features between various DRL algorithms, using a different imitation learning algorithm can hardly recover the model with high fidelity, as we will evaluate in Section \ref{sec:exp}.

To address the above issue, we propose a new methodology to identify the training algorithm family (i.e., the hypothesis class) of a black-box model. The key idea of this approach is that DRL models trained from different algorithms can yield different behaviors from the same states. Such temporal sequences can be distinguished by machine learning classifiers, which can leak the training algorithm to the adversary from the model's dynamics.

%We have shown that the adversary can adopt imitation learning techniques to extract DRL models when the hypothesis classes are the same. However, the adversary has no prior knowledge about the hypothesis class for the targeted model. To address this issue, our design strategy is to identify the training algorithm family of the targeted model that almost determines the hypothesis class. 
% the reason of why we use RNN classifier In other words, distinguishing the temporal dynamic behavior of different training algorithm families is one of the most critical tasks for our DRL model extraction attack. We observe that RNN can be used to approximate partially observable MDPs, which meets the requirements to distinguish temporal sequences generated by different training algorithm families.

Hence, we present our end-to-end attack approach that consists of two stages. At the first stage, we construct a classifier, which can be used to predict the \emph{training algorithm family} of a given black-box DRL model from its runtime behaviors. At the second stage, based on the extracted information, we adopt state-of-the-art imitation learning techniques to generate a model with similar \emph{behaviors} as the victim one. Figure \ref{fig:algorithm} illustrates the methodology overview, and Algorithm \ref{algo:algorithm} describes the detailed steps. 

\begin{algorithm}[htb]
\SetAlgoLined
% \SetAlgoNoLine  % remove vertical lines in for or while
\SetNoFillComment
\LinesNumbered
\caption{Extracting DRL models}
\label{algo:algorithm}
\KwIn{Targeted model $M^*$, DRL environment $env$}
\KwOut{Replicated model $M'$}
\tcc{Stage 1}
Set up a set of random seeds $\mathbb{S}$, reward threshold $R$, action sequence length $T$ \;
Select algorithm family pool $\mathbb{P}$; %hyperparameters pool $\mathbb{P}$ \;
Dataset $\mathbb{D} = \emptyset$ \;
 \For{\emph{each} p $\in \mathbb{P}$}{\label{line:for_begin}
%  \For{\emph{each} p $\in \mathbb{P}$}{
   \For{\emph{each} s $\in \mathbb{S}$}{
    $env$.\texttt{initialize($s$)}\;
    $m$ = \texttt{train\_DRL}($env$, $p$)\;
    \If{\texttt{\emph{eval}}($m$, $env$) $>$ $R$}
    {$A$ = \texttt{GenSequence}($m$, $env$, $T$)\;
     $\mathbb{D}$.\texttt{add}([$A$, $p$])\;
    }
   }
%  }
 }\label{line:for_end}
 $C$ = \texttt{train\_RNN}($\mathbb{D}$)\;

\tcc{Extract algorithm family}
$A^*$ = \texttt{GenSequence}($M^*$, $env$, $T$)\;
$P^*$ = $C$.\texttt{predict}($A^*$)

\tcc{Stage 2}
$M'$ = \texttt{ImitationLearning}($M^*$, $P^*$, $env$)\;
\While{\texttt{\emph{eval}}($M'$, $env$) $<$ \texttt{\emph{eval}}($M^*$, $env$)}{
$M'$ = \texttt{ImitationLearning}($M^*$, $P^*$, $env$)\;
}
\KwRet{$M'$}
\end{algorithm}

% \subsection{Extracting DRL Model Algorithm Families via RNN Classification}
\subsection{Identifying Training Algorithm Families} \label{ssec:stage_1}
As the first stage, we train a classifier, whose input is a DRL model's action sequence, and output is the model's algorithm family. With this classifier, we can identify the algorithm family of an arbitrary black-box DRL model.

\noindent{\textbf{Dataset preparation.}} 
A dataset is necessary to train this classifier. 
It should include action sequences of DRL models trained from different algorithms. To this end, we first train a quantity of shadow DRL models in the same environment but with various algorithms, and then collect their behaviors to form this dataset. 

Specifically, we set up an algorithm family pool $\mathbb{P}$ that includes all the DRL algorithm families in our consideration. We also prepare a set $\mathbb{S}$ of random seeds for environment initialization. For each algorithm family in $\mathbb{P}$, we train some DRL models in the environments initialized by different random seeds in $\mathbb{S}$. We evaluate the performance of each trained DRL model by comparing its reward with a reward threshold $R$: we only select the DRL models whose rewards are higher than $R$. For each qualified model, we collect $N$ different state-action sequences with a length of $T$: $\{(s_1, a_1), (s_2, a_2), ...(s_T, a_T)\}$. Then samples are generated with the action sequences ($A=\{a_1, a_2, ...a_T\}$) as the feature and the training algorithm family as the label, to construct the dataset.

\noindent{\textbf{Training.}}
We adopt the Recurrent Neural Network (RNN) as the classifier\footnote{In Section \ref{sec:exp} we also implemented and tested other network structures as well. RNN gives the best results for recognizing sequential features.}. An RNN is competent of processing sequence data of arbitrary lengths by recursively applying a transition function to its internal hidden state vector of the input. It is generally used to map the input sequence to a fixed-sized vector, which will be further fed to a softmax layer for classification. However, vanilla RNNs suffer from the gradient vanishing and exploding problem: during training, components of the gradient vector can grow or decay exponentially over long sequences. To address this problem, we adopt the Long Short-Term Memory (LSTM) network \cite{hochreiter1997long} in our design. LSTMs can selectively remember or forget things regulated by a set of gates. Each gate has a sigmoid neural net layer and a pointwise multiplication operation, which can filter the information through the network. Therefore, LSTM units can maintain information in memory for a long period under the control of the gates.

To train the classifier over the prepared dataset, for each input sequence $A = \{a_1, a_2, ... , a_T\}$, we first apply a set of LSTM layers to obtain its vector representation. Then we attach a fully-connected layer and a non-linear softmax layer after the LSTMs to output the probability distribution over all classes of possible model algorithms. We use cross-entropy loss to identify the optimal parameters for this classifier by minimizing the loss function.

\noindent{\textbf{Extracting algorithm families.}}
With this RNN classifier, we are now able to predict the training algorithm family of a given back-box DRL model. We operate this targeted model in the same environment with certain random seed and collect the action sequence for $T$ rounds. We query the classifier with this sequence and get the probability of each candidate algorithm family. We select the one with the highest probability as the attack result. To further increase the confidence and eliminate the stochastic effects, we can run the targeted model in different initialized environments and collect the sequences for predictions. We choose the most-predicted label as the targeted model's training algorithm family. 

% \subsection{Replicating DRL Models via Imitation Learning} 
\subsection{Extracting Models via Imitation Learning}
\label{stage_2}

With the extracted training algorithm family, the adversary can now perform a $(\delta, \xi)$-DRL extraction attack via imitation learning. Various techniques have been designed to imitate the behaviors of DRL models. For instance, inverse optimal control was adopted in \cite{finn2016guided} to learn the behaviors from expert's demonstrations. Inverse Reinforcement Learning (IRL) \cite{russell1998learning} was then proposed to improve the learning effects when no clear reward function is available. Generative Adversarial Imitation Learning (GAIL) \cite{ho2016generative} was proposed to directly learn policies using Generative Adversarial Framework. Deep Q-learning from Demonstration (DQfD) \cite{hester2018deep} combined temporal difference updates with supervised classification of the expert's actions to accelerate the training of DRL models.

%, such as the task of driving. However, IRL only learns the reward function of the environment based on the expert’s demonstrations, which explains expert behavior but does not directly tell the learner how to act. To address this issue, 

%the adversary can train a new model (or just pick a shadow DRL model during the classifier training step) from the same algorithm family as the replica of the targeted model. However, due to the complexity of DRL algorithms and variance of initial environment, this replicated model can still exhibit distinct behaviors from the real one, even they are from the same algorithm family. To achieve $\delta$-DRL extraction attack in Definition \ref{def:attack}, this stage aims to refine the replicated model via sequential imitation learning. Fortunately, several works have introduced model imitation on DRL models, e.g., GAIL\cite{ho2016generative} and 

In our implementation, we adopt the GAIL framework for model extraction, as it can be easily converted to an $\epsilon_{\delta}$-imitation learning algorithm. 
GAIL is a model-free method that can obtain significant performance gains in imitating complex behaviors in large-scale and high-dimensional environments. It generalizes Inverse Reinforcement Learning \cite{russell1998learning} by formulating the imitation learning as minimax optimization, which can be solved by alternating gradient-type algorithms in a more scalable and efficient manner.

% ,  GAIL aims to learn a policy $\pi$ from the expert model by solving the following minimax optimization problem,
% \begin{equation}\label{equ: gail_optimization_function}
%     \min_\pi\max_{r\in\emph{R}} \mathbb{E}_\pi[r(s,a)] - \mathbb{E}_{\pi_n^*}[r(s,a)],
% \end{equation}
% where $\mathbb{E}_\pi[r(s,a)] = \lim_{T\to\infty} \mathbb{E}[\tfrac{1}{T}\sum_{t=0}^{T-1}r(s_t,a_t)|\pi]$ represents the average reward under the policy $\pi$ when the reward function is $r$, and $\mathbb{E}_{\pi_n^*}[r(s,a)] = \tfrac{1}{nT}\sum_{i=1}^n\sum_{t=0}^{T-1}r(s_t^i,a_t^i)$ denotes the empirical average reward over the demonstration trajectories. As shown in Equation \ref{equ: gail_optimization_function}, GAIL tries to find a policy which attains an average reward similar to that of the expert policy with respect to any reward belonging to the reward function class $\emph{R}$. GAIL is different from Inverse Reinforcement Learning because it learns the policy, not the reward function, directly from the data. 

Specifically, similar as Generative Adversarial Networks, GAIL is composed of two neural networks, which contest with each other during the imitation process: a generative DRL model $G$, and a discriminative model $D$ whose job is to distinguish the distribution of data generated by the generator $G$ from the ground-truth data distribution of the expert DRL model.
Given a sequence of states $S_i$ and the corresponding actions $A_i$, $D$ outputs the probability of the sequences generated by the expert model. During the training process, $D$ is optimized to accurately distinguish the state-action sequences generated by $G$ and the expert model. In particular, GAIL updates $D$ iteratively using the gradient ascent technique to minimize the loss function:
\begin{equation}
    L_D = - \mathbb{E}_{G}[\log (D(S_i, A_i))] - \mathbb{E}_{\pi^*}[\log (1 - D(S_i, A_i^*))],
\end{equation}
where $S_i, A_i$ are generated by $G$, while $S_i, A_i^*$ are produced by the expert model $\pi^*$.

% For a sampled trajectory, which is consist of a sequence of state-action pairs $(s_1, a_1),..., (s_T, a_T)$ (i.e., $(S_i, A_i)$), the discriminator $D$ will give a value $D(S_i, A_i)$ to indicate whether the trajectory is sampled from the expert or the generator. For two trajectories $(S_i^*, A_i^*)$ and $(S_i, A_i)$ sampled from expert model and the generator, the discriminator $D$ will give a higher value on the expert trajectory (i.e., $D(S_i^*, A_i^*) > D(S_i, A_i)$). To distinguish the trajectories, the discriminator updates its parameters iteratively with gradient ascent technique on the loss function $L_D$ which can be defined as follow,
% \begin{equation}
%     L_D = - \mathbb{E}_{\pi}[\log (D(S_i, A_i))] - \mathbb{E}_{\pi^*}[\log (1 - D(S_i^*, A_i^*))].
% \end{equation}
% \begin{equation}
%     L_D = - \sum_{i=1}^n\log (D(S_i, A_i)) - \sum_{i=1}^n\log (1 - D(S_i^*, A_i^*)).
% \end{equation}

A satisfactory generator $G$ should behave similar as the expert model. To this end, GAIL requires $G$ to increase the probabilities of $D$ on the sequences it generates and optimize $G$ by minimizing the following loss function,
\begin{equation}
        L_G = - \sum_{i=1}^n D(S_i, A_i) P(S_i, A_i|G),
\end{equation}
where $P(S_i, A_i|G)$ is the probability that the sequences $S_i, A_i$ occur given the generator $G$. 
% From the analyze above, we know that a trajectory $\tau_i$ is consist of a sequence of state-action pairs $(S_i, A_i)$ and the value of $D(\tau))$ is determined by these state-action pairs. 

Then the overall optimization goal of GAIL is to minimize the Jensen-Shannon divergence between the behaviors of $G$ and the expert model \cite{ho2016generative}:
\begin{equation}
    \argmin_G \argmax_D L(G, D) = \argmin_G d_{JS}(\rho_G || \rho_{\pi^*}) - R(G)
\end{equation}
where $R(\pi_{G})$ is the policy regularizer and $d_{JS}(p||q) = d_{KL}(p||(p+q)/2) + d_{KL}(q||(p+q)/2)$ represents the Jensen-Shannon divergence between two distributions $p$ and $q$. $\rho_{\pi^*}$ is a joint distribution of states and actions following $\pi$. 
% Therefore, in the learning progress, GAIL will try to optimize the generator to behave as similar to the expert's behavior as possible.
During the learning process, GAIL alternately trains the generative and discriminative models until $d_{JS}(\rho_G || \rho_{\pi^*})$ is smaller than $\epsilon_{\delta}$.

Considering the stochasticity of the imitation learning process, it is possible that the model cannot reach the same reward although it has the same behaviors as the targeted model. Therefore, we repeat the GAIL process until a qualified model is obtained which has very similar performance (i.e., reward) as the targeted model.

\section{Evaluation}\label{sec:exp}
\subsection{Implementation and Experimental Setup}

Our attack approach is general and applicable to various reinforcement learning tasks. Without loss of generality, we consider two popular environments: Cart-Pole and Atari Pong \cite{baselines}. For each environment, we train DRL models with five mainstream DRL algorithm families (DQN \cite{mnih2015human}, PPO \cite{schulman2017proximal}, ACER \cite{wang2016sample}, ACKTR \cite{wu2017scalable} and A2C \cite{mnih2016asynchronous}). We believe these two tasks with five algorithms are representative, as they are commonly used in academia for evaluating reinforcement learning applications. These algorithms also demonstrate great potential in real-world products. All the model configurations and hyperparameters follow the default setup in the OpenAI Baselines framework \cite{baselines}.

For the algorithm classification, we select 50 shadow models from each DRL algorithm and environment whose rewards are higher than a task-specific threshold $R$, defined in the OpenAI framework. We consider different sequence lengths $T$ (50, 100 and 200), and compare their impacts on the prediction accuracy. For each shadow DRL model, we collect 50 action sequences as the training inputs of our classifier. Therefore, for Cart-Pole and Atari Pong, the sizes of the datasets from 250 shadow models are both 12,500. To evaluate the classifier, we randomly split the trajectory data to training and test sets with a ratio 8:2.
During the training process, the initial learning rate is set to 0.005 with a decay factor of 0.7 when loss plateaus occur. The batch size is set to 32.
% Without loss of generality, we set $\xi = \epsilon_{\delta}$ and stop the training after $N=100$ iterations.
Without loss of generality, we set $\xi = \epsilon_{\delta}$. Instead of predefining $\epsilon$ and the corresponding $\epsilon_{\delta}$ for each task, we stop the training after $N=100$ iterations in.

% \subsection{Results of RNN Classification} 
\subsection{Results of Algorithm Family Identification}
\label{ssec:result}

\noindent{\textbf{Comparisons of different classification models.}}
As the first study, we investigate the effectiveness of different machine learning models on predicting DRL training algorithm families. We implement the classifiers based on various common models, i.e., RNN, MLP, SVM and random forest. Their prediction accuracy for each algorithm family as well as the average values are shown in Table \ref{tab:different_classifiers}. We observe the RNN classifier achieves the highest accuracy, followed by the MLP classifier. The performance of SVM and random forest is extremely low, since they are not good at handing sequential inputs. Therefore, RNN is the optimal choice for classification, and will be used for the rest experiments. 

\begin{table}[t]
    \centering
    % \resizebox{0.9\linewidth}{!}{
    \begin{tabular}{ >{\centering}m{0.08\textwidth} | >{\centering}m{0.06\textwidth}
    >{\centering}m{0.06\textwidth}
    >{\centering}m{0.06\textwidth}
    >{\centering\arraybackslash}m{0.12\textwidth}}
    \Xhline{1pt}
    Algorithm & RNN & MLP & SVM & Random Forest \\ 
    \Xhline{1pt}
    A2C    & 82\% & 62.4\% & 74.8\% & 34.8\% \\
    PPO    & 99\% & 79.4\% & 79.6\% & 77.8\% \\
    ACER   & 54\% & 41.8\% & 27.6\% & 28.6\% \\
    ACKTR  & 82\% & 67.6\% & 28.6\% & 44.8\% \\
    DQN    & 95\% & 47\%   & 21\%   & 37.6\% \\
    \Xhline{1pt}
    Average & 82.4\% & 60\% & 46.3\% & 44.6\% \\
    \Xhline{1pt} 
    \end{tabular}
    % }
    \caption{The accuracy of different classifiers}
    \label{tab:different_classifiers}
    \vspace{-15pt}
\end{table}

\noindent\textbf{Impact of hyperparameters for the algorithm classifier.}
The accuracy of the RNN classifier can be affected by a few hyperparameters (the length of the input sequence, the number of hidden layers). Figure \ref{fig:accuracy_impact} shows the accuracy under different combinations of these settings. First, we observe that the length of the input sequence can affect the classification performance: a longer input sequence can give a higher accuracy. Therefore, for Cart-Pole environment, we take all the actions within one episode as the input sequence ($T$=200). For Atari Pong environment, one episode can have up to 10,000 actions. It is not recommended to take the entire episode as input, which can incur very high cost and training over-fitting. Since $T=200$ can already give us very satisfactory accuracy, we will set the length of input sequence to 200 as well.

\setlength{\intextsep}{5pt}%
\setlength{\columnsep}{3pt}%
\begin{wrapfigure}{r}{0.6\columnwidth}
    \vspace{-10pt}
    \includegraphics[width=0.6\columnwidth]{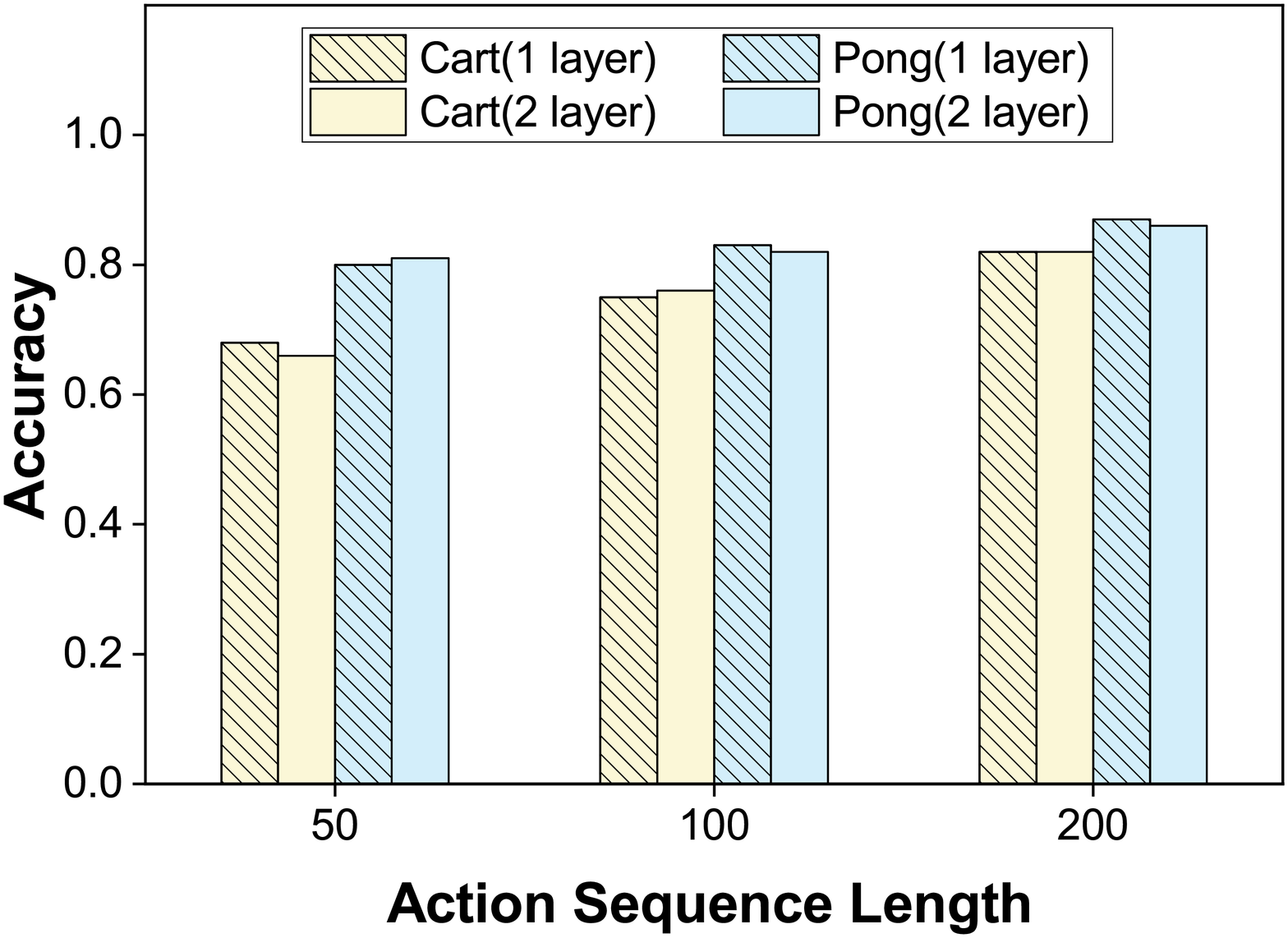}
    \vspace{-20pt}
    \caption{Average accuracy of RNN}
    \vspace{-5pt}
    \label{fig:accuracy_impact}
\end{wrapfigure}

Second, we consider different numbers of hidden LSTM layers (1 and 2) for the classifier. We observe that this factor has slight influence on the accuracy of the classifier. One hidden layer can already validate the effectiveness of the RNN classification. So in the following experiments, we will adopt a 1-layer RNN for simplicity. 

Third, the action space can also affect the classification accuracy. Higher-dimensional actions contain more information about the DRL model, and can be identified more accurately. In our case, the action space of Cart-Pole environment is 2  while that of Atari Pong environment is 6. Then the classification of Atari Pong has a higher accuracy than Cart-Pole, as reflected in Figure \ref{fig:accuracy_impact}.

\begin{figure}[t]
\centering
    \begin{minipage}{.49\columnwidth}
        \includegraphics[width=\textwidth]{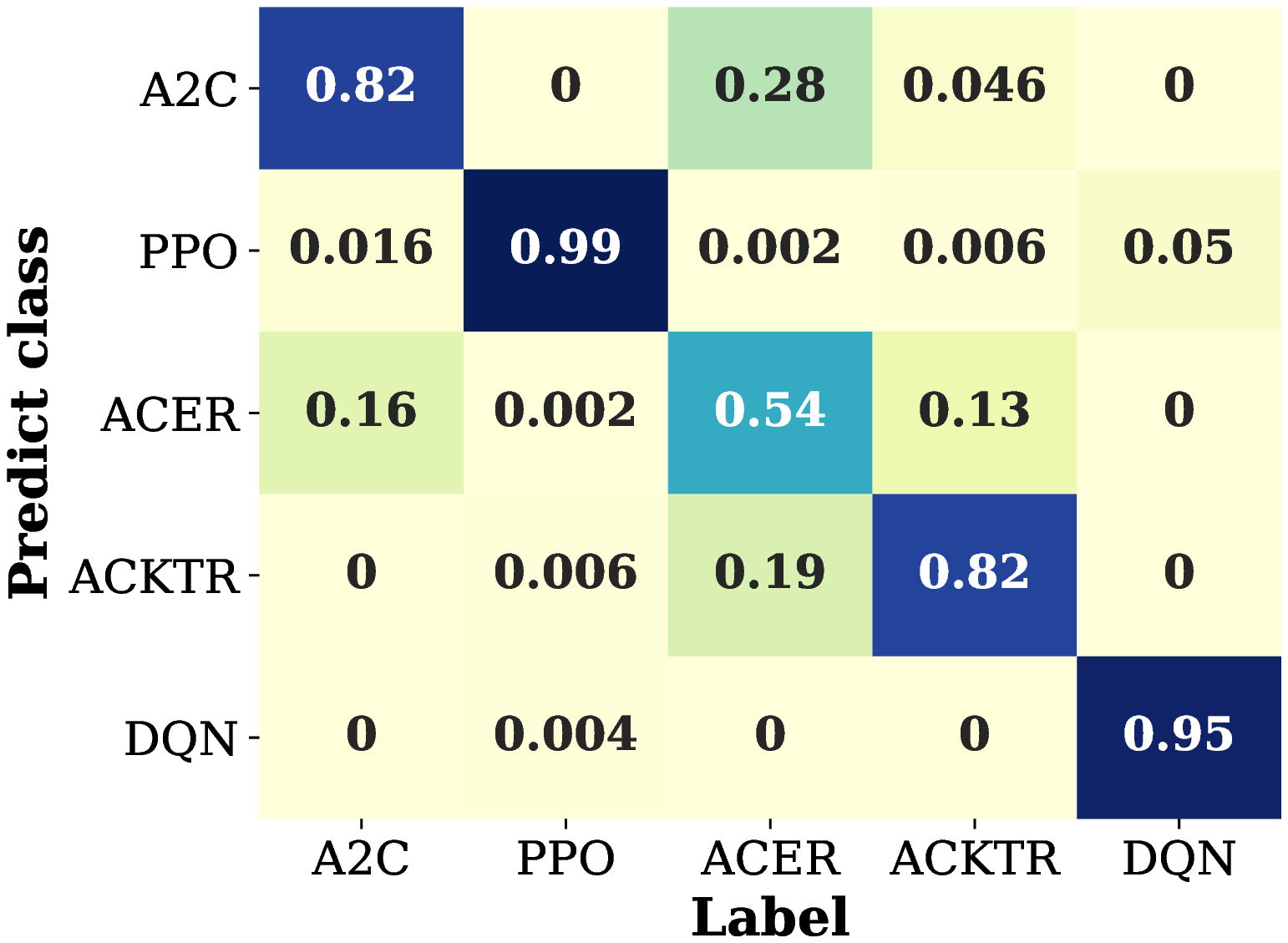}
        \subcaption{Cart-Pole}
        \label{subfig:accuracy_heatmap_cartpole}
    \end{minipage}
    \hfill
    \begin{minipage}{.49\columnwidth}
        \includegraphics[width=\textwidth]{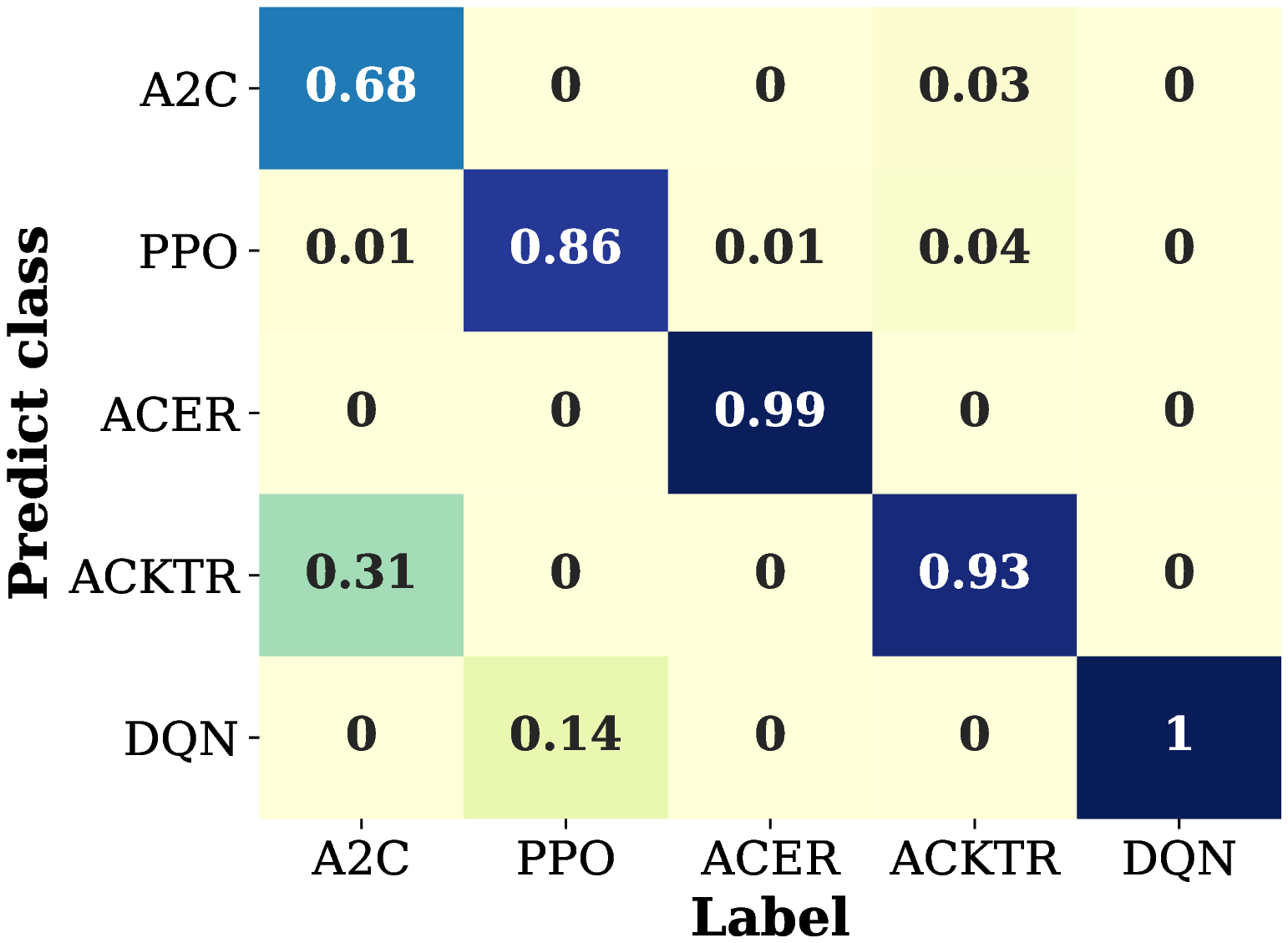}
        \subcaption{Atari Pong}
        \label{subfig:accuracy_heatmap_pong}
    \end{minipage}
    \caption{The accuracy of identified algorithm families}
    \label{fig:accuracy}
    % \vspace{-3pt}
\end{figure}

\noindent{\textbf{Accuracy of each class.}}
Figures \ref{fig:accuracy} shows the confusion matrix for both environments. We observe the classifier can distinguish DRL models of each algorithm family with very high confidence. For most cases, the prediction accuracy is above 70\%; the best case is 100\% (DQN models in Atari Pong); the worst case is 54\% (ACER models in Cart-Pole), which is still much higher than random guess (20\%). The prediction accuracy of the DQN model is particularly high (95\% in Cart-Pole, and 100\% in Atari Pong). This is because DQN is a value-based algorithm while all the other algorithms are actor-critic methods. So DQN models are easier to be distinguished.

\begin{table}[t]
    \centering
    \begin{tabular}{>{\centering}m{0.075\textwidth}|
                >{\centering}m{0.06\textwidth}
                >{\centering}m{0.05\textwidth}
                >{\centering}m{0.05\textwidth}
                >{\centering}m{0.06\textwidth}
                >{\centering\arraybackslash}m{0.07\textwidth}}
    \Xhline{1pt}
    \multirow{2}{*}{Algorithm}
    & \multirow{2}{*}{Default}
    & \multicolumn{2}{c}{Learning Rate}
    & Train step 
    & Structure \\
    \cline{3-6} & & +50\% & -20\%  & +100\%  &  +1 Layer\\
    \Xhline{1pt}
    A2C    & 82\%    & 73\%   & 86\%     & 71\%   & 72\%  \\
    PPO    & 99\%    & 98\%   & 91\%     & 83\%   & 65\%  \\
    ACER   & 54\%    & 58\%   & 52\%     & 71\%   & 59\%  \\
    ACKTR  & 82\%    & 86\%   & 91\%     & 87\%   & 66\%  \\
    DQN    & 95\%    & 80\%   & 72\%     & 88\%   & 96\%  \\
    \Xhline{1pt}
    Average & 82.4\% & 79\%   & 78.4\%   & 80\%   & 71.6\%  \\
    \Xhline{1pt}
    \end{tabular}
    \caption{The impact of hyperparameters and network structures in the same family}
    \label{tab:hyperparameters}
    \vspace{-15pt}
\end{table}

\noindent{\textbf{Impact of hyperparameters and structures for DRL models.}}
When training the classifier, we use the default hyperparameters and model structures in the OpenAI Baselines benchmark \cite{baselines} to generate shadow DRL models and collect the sequence features. However, our classifier is still able to predict the algorithm family of a DRL model which uses a different set of hyperparameters and network structures, as the training algorithm dominates the network configurations in reflecting the model's characteristics. 

To confirm this, we train targeted DRL models with different hyperparameter values and model structures, and use our classifier to predict their algorithms. Specifically, considering the unstable learning process of large hyperparameters variance, we alter the learning rates (i.e., 150\% and 80\% of the default one) and training steps (double of the default one) properly in the Cart-Pole model training progress. 
The default network structures of all the algorithm families in our consideration consist of two hidden layers with 64 neurons, with an exception of DQN which has only one hidden layer. We modify the model structures by adding one more hidden layer to each network.
The prediction results are shown in Table \ref{tab:hyperparameters}. We observe that the prediction accuracy is still very high when the hyperparameters and structures are different from what are used during training. The results indicate that our classifier can learn the characteristics of different algorithm families instead of just memorizing certain samples.

% The accuracy of 50\% learning rate is relatively lower, as such small learning rate will produce unstable learning progress of DRL models, which is not recommended.

% \iffalse
% \begin{table}[t]
%     \centering
%     \begin{tabular}{>{\centering}m{0.075\textwidth}|
%                 >{\centering}m{0.06\textwidth}
%                 >{\centering}m{0.04\textwidth}
%                 >{\centering}m{0.045\textwidth}
%                 >{\centering}m{0.045\textwidth}
%                 >{\centering\arraybackslash}m{0.065\textwidth}}
%     \hline
%     \multirow{2}{*}{Algorithm}
%     & \multirow{2}{*}{Default}
%     & \multicolumn{3}{c}{Learning Rate}
%     & Training step \\
%     \cline{3-5} & & +50\% & -20\% &-50\%  & +100\%\\
%     \hline
%     A2C    & 82\%    & 73\%   & 86\%    & 71\%    & 71\%  \\
%     PPO    & 99\%    & 98\%   & 91\%    & 81\%    & 83\%  \\
%     ACER   & 54\%    & 58\%   & 52\%    & 31\%    & 71\%  \\
%     ACKTR  & 82\%    & 86\%   & 91\%    & 72\%    & 87\%  \\
%     DQN    & 95\%    & 80\%   & 72\%    & 85\%    & 88\%  \\
%     \hline
%     Average & 82.4\% & 79\%   & 78.4\%  & 68\%   & 80\%  \\
%     \hline
%     \end{tabular}
%     \caption{The impact of hyperparameters in the same family}
%     \label{tab:hyperparameters}
% \end{table}
% \fi

\begin{figure}[t]
    \centering
    \includegraphics[width=0.5\textwidth]{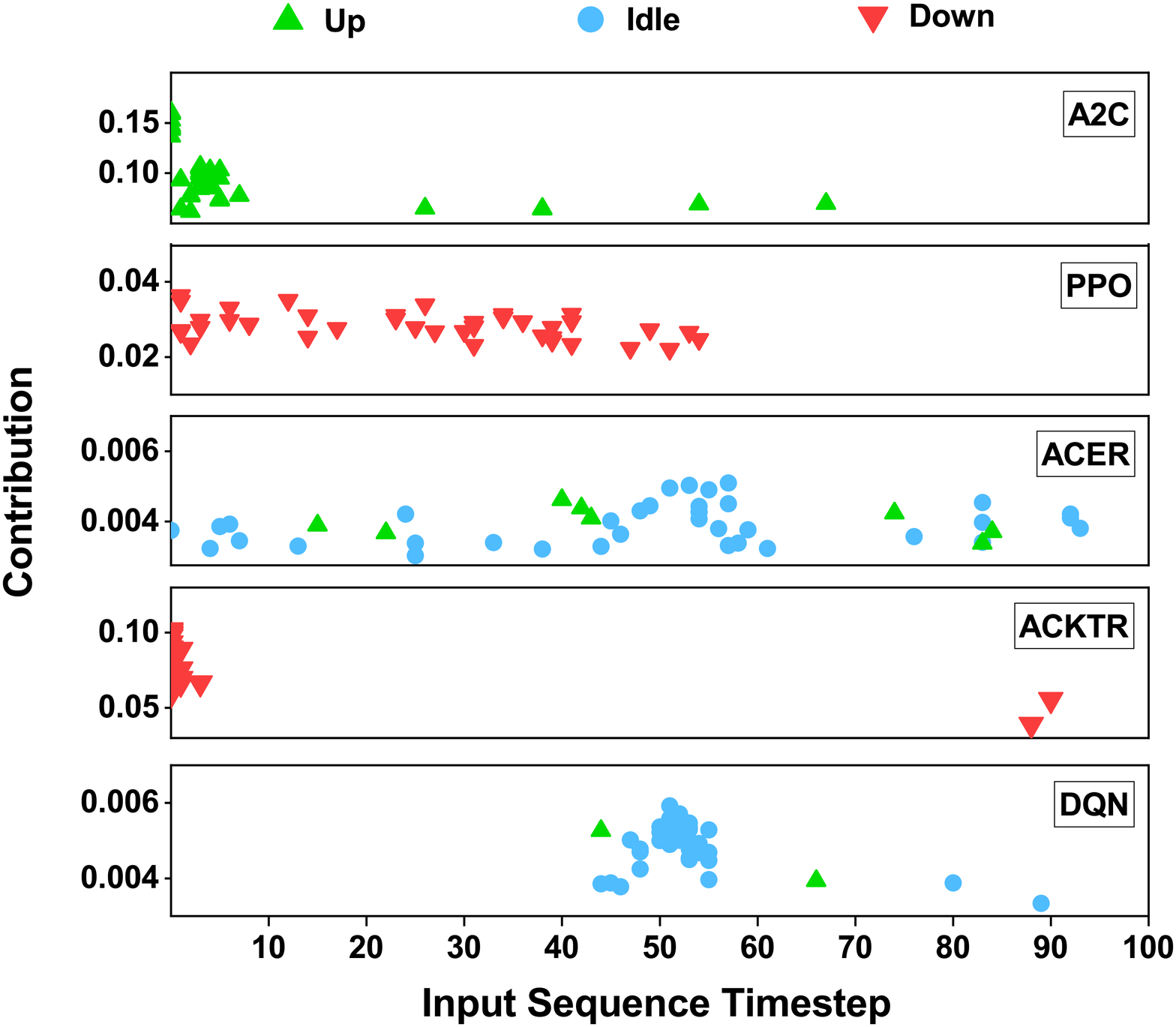}
    \caption{Action preference of different algorithm families}
    \label{fig:explain}
    % \vspace{-20pt}
\end{figure}

\noindent\textbf{Interpretation of the classification results.}
We quantitatively explain why our classifier can distinguish different DRL algorithm families. We adopt the Local Interpretable Model-agnostic Explanations (LIME) framework \cite{lime}, which understands a model by perturbing its input and observing how the output changes. Specifically, it modifies a single data sample by tweaking the feature values and observes the resulting impact on the output to determine which features play an important role in the model predictions.

In our case, we build an explainer with LIME on our RNN classifier. Then, we randomly select 200 explanation instances from the training data of the classifier in Atari Pong environment. We feed these instances to the explainer and obtain the explanation results. For each explanation instance, we identify the feature (i.e., one action of the input sequence) which contributes most to the prediction. Through the analysis of these features, we can discover the different behaviors of DRL models trained from different algorithms.
Figure \ref{fig:explain} shows the contribution of the actions (\texttt{UP}, \texttt{DOWN}, \texttt{IDLE}) with prominent impacts on the prediction in each input sequence. We can observe that models trained with different DRL algorithm families give very different behavior preferences. A2C tends to issue important actions of \texttt{UP} at the beginning of the sequence; ACKTR prefers to give the action of \texttt{DOWN} also at the beginning of the task; DQN has a higher chance to predict \texttt{IDLE} clustering at the beginning; PPO issues the \texttt{DOWN} action all over the sequence with a large variance of contribution factor; ACER has important actions of \texttt{UP} and \texttt{IDLE} with similar contributions spanning all over the sequence. This shows those DRL algorithms have quite different characteristics in making action decisions, giving the classifier an opportunity to distinguish them just based on the actions.

% \begin{figure*}[t]
% \centering
%     \begin{minipage}{.33\textwidth}
%         \includegraphics[width=\textwidth]{imitaion_performance.eps}
%         \caption{Imitation model performance}
%         \label{subfig:imitation_reward}
%     \end{minipage}
%     \hfill
%     \begin{minipage}{.33\textwidth}
%         \includegraphics[width=\textwidth]{divergence.eps}
%         \caption{Imitation model divergence}
%         \label{subfig:model_similarity}
%     \end{minipage}
%     \hfill
%     \begin{minipage}{.33\textwidth}
%         \includegraphics[width=\textwidth]{adversarial_pong.eps}
%         \caption{Enhancing transferability}
%         \label{subfig:transferability}
%     \end{minipage}
% % \caption{The results of imitation attack and adversarial attack}
% % \label{fig:accuracy}
% \end{figure*}

\begin{figure}[t]
\centering
    \begin{minipage}{.49\columnwidth}
        \includegraphics[width=\textwidth]{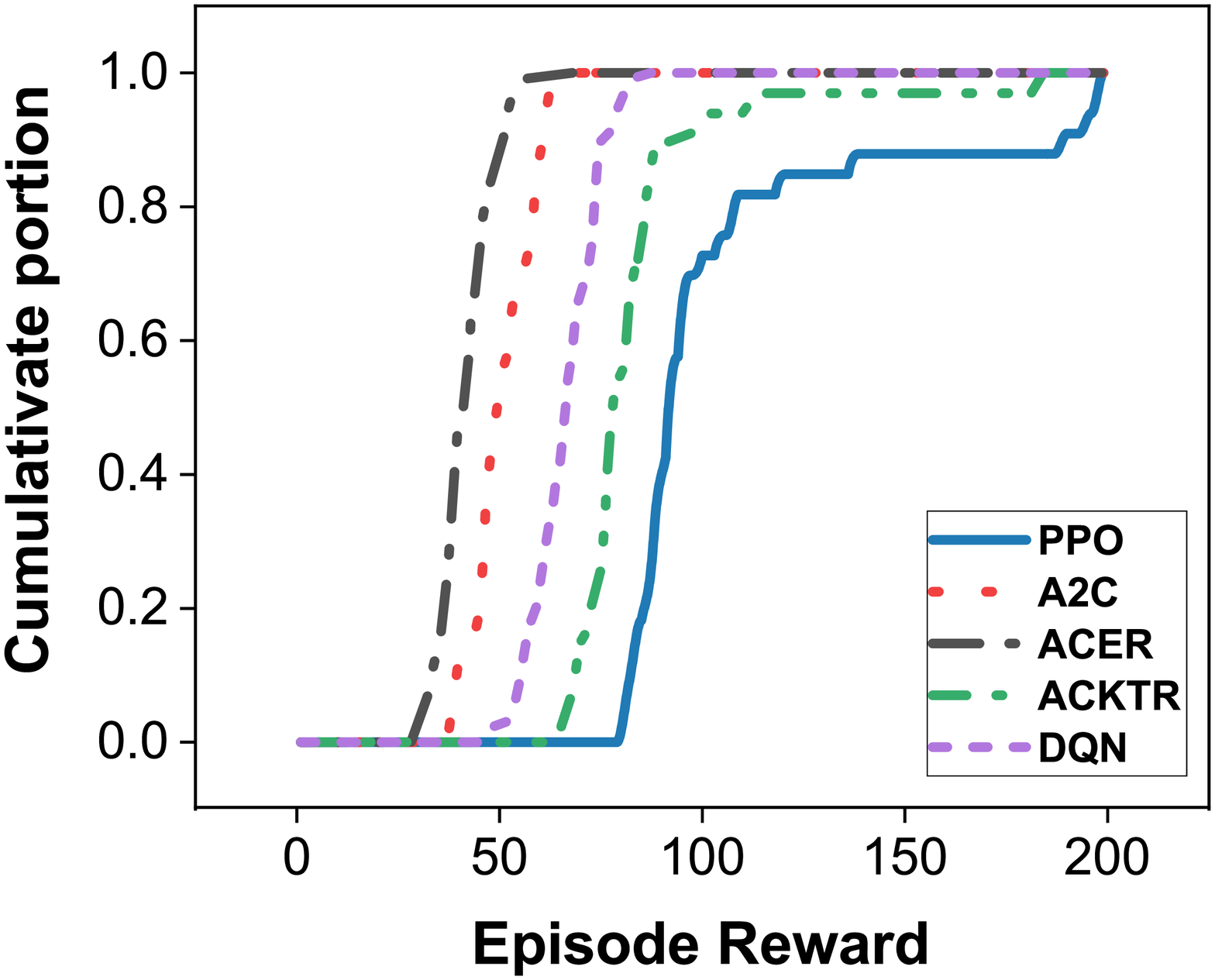}
        \subcaption{Performance accuracy (PPO)}
        \label{subfig:ppo_imitation_reward}
    \end{minipage}
    \begin{minipage}{.49\columnwidth}
        \includegraphics[width=\textwidth]{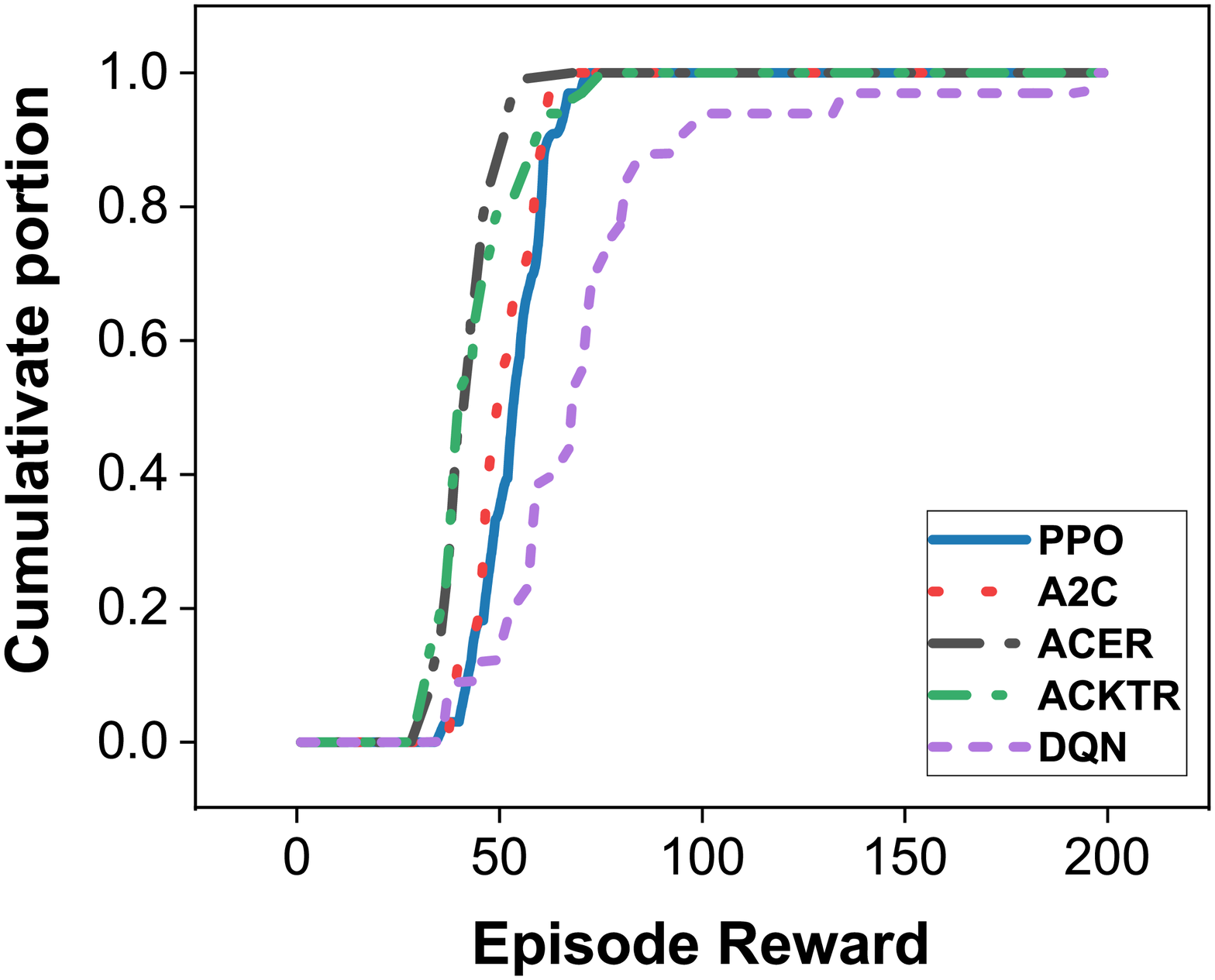}
        \subcaption{Performance accuracy (DQN)}
        \label{subfig:dqn_imitation_reward}
    \end{minipage}
    \caption{Results of accuracy extraction}
    \label{fig:accuracy-extract}
    \vspace{-9pt}
\end{figure}

% \subsection{Results of Imitation Learning}
\subsection{Results of Imitation Learning}
We demonstrate the effectiveness of imitation learning for model extraction. To train the replicated models, GAIL is applied to imitate the behaviors of the targeted model. 
We use two different types of DRL algorithms (i.e., value-based algorithm DQN and actor-critic algorithm PPO) as the generators of GAIL. Other policies can be applied in the same way. We follow \cite{baselines,lennon2019modeling} to implement the PPO and DQN generators of GAIL, respectively. Without loss of generality, we evaluate the imitation effects in Cart-Pole environment.

% As implemented in OpenAI baseline \cite{baselines}, the generator of GAIL can be PPO or TRPO policies. We select PPO as the generator in this paper, and other policies can be applied in the same way. 

\noindent{\textbf{Accuracy extraction.}}
The extracted model with the same algorithm family can reach similar performance (i.e., reward) as the targeted model after imitation learning.  We assume the targeted model is in the PPO and DQN family, and the adversary has extracted such information via RNN classification. Then he adopts the GAIL framework to learn a new model from the same family. We repeat this process 30 times, and Figure \ref{fig:accuracy-extract} shows the cumulative probability of the episode reward of all extracted models (blue solid line in Figure \ref{subfig:ppo_imitation_reward} and purple dash line in Figure \ref{subfig:dqn_imitation_reward}). We can observe that the extracted model has a big chance to reach high episode rewards as the targeted model. In contrast, when the adversary uses a different training family from the targeted model for imitation learning, the extracted model can rarely reach high rewards. This indicates the importance of the algorithm identified by the RNN classification, in order to perform high-quality imitation learning. 

% (e.g., in Figure \ref{subfig:ppo_imitation_reward}, the colored dash lines, where the adversary uses PPO, while the targeted model is A2C, ACER, ACKTR and DQN respectively)

To further study the performance of the extracted models imitated with the same algorithm as the targeted model, we demonstrate the learning process (both models adopt PPO) in Figure \ref{subfig:imitation_ppo}. We observe that in the first imitation cycle, the extracted model cannot reach the same reward as the targeted one, as it learns the random behaviors of the targeted model with low rewards. Then we start a new imitation cycle, and now the learned model can get the same reward as the victim model (i.e., 100\%). We can stop with this replica, or continue to identify more qualified ones (at the 6th cycle). In contrast, we also consider a case where the adversary does not know the training algorithm, and randomly pick one for imitation learning. Figure \ref{subfig:imitation_acer} shows the corresponding imitation process (the targeted model uses ACER while the adversary selects the PPO generator). Now the extracted model can never get the same reward as the targeted model.

\begin{figure}[t]
\centering
    \begin{subfigure}{\columnwidth}
        \includegraphics[width=\textwidth]{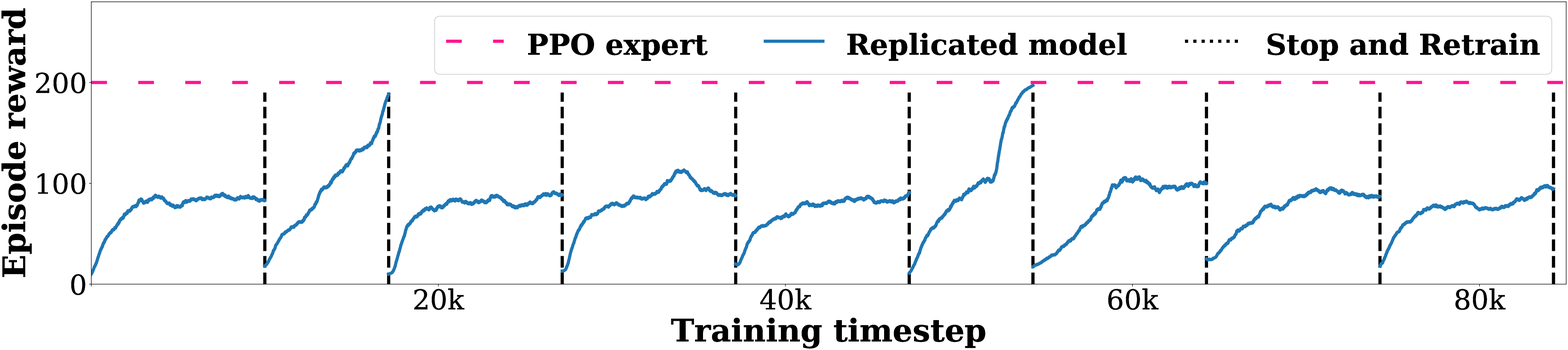}
        \caption{Imitation learning with the same algorithm}
        \label{subfig:imitation_ppo}
        \vspace{10pt}
    \end{subfigure}
    \begin{subfigure}{\columnwidth}
        \includegraphics[width=\textwidth]{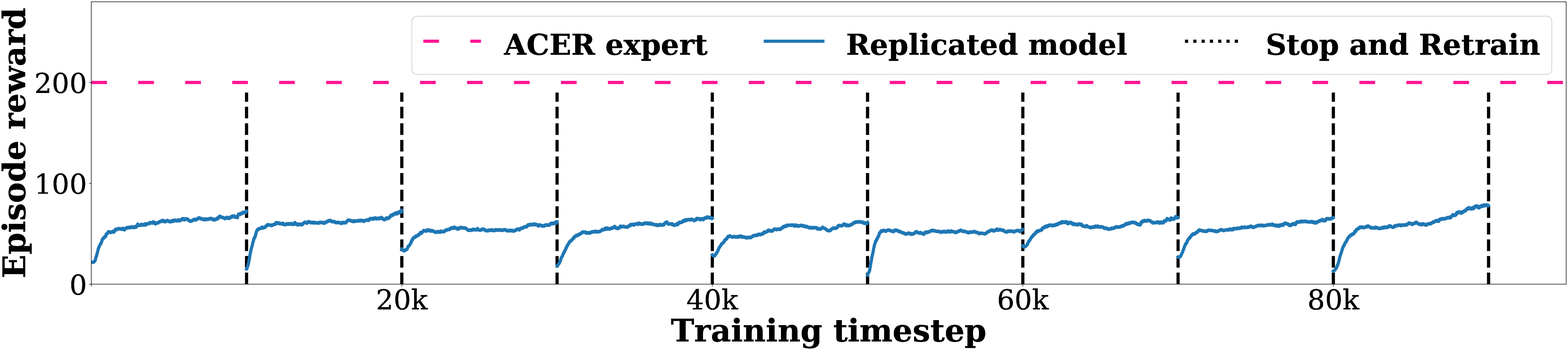}
        \caption{Imitation learning with different algorithms}
        \label{subfig:imitation_acer}
    \end{subfigure}

\caption{The reward curve of the extracted model during the imitation learning process}
\label{fig:imitation}
% \vspace{-20pt}
\end{figure}

\noindent{\textbf{Fidelity extraction.}}
In addition to the rewards, the extracted model can also learn similar behaviors as the targeted one. Since the output of a DRL model is a probability distribution over legal actions, we adopt the Jensen-Shannon (JS) divergence \cite{lin1991divergence} to measure the similarity of the action probability distributions between the imitated and the targeted models. We consider three cases: (1) the similarity between the targeted model and itself (i.e., collecting the behaviors twice). This serves as the baseline for comparison. (2) the similarity between the targeted model and the extracted model from imitation learning; (3) the similarity between the targeted model and a shadow model which is trained from scratch with the same algorithm rather than imitation learning. For each case, we feed the same states to the two models in comparison, sample 100 actions from each model, compute the action probability distributions and the divergence between these two action distributions. Figure \ref{fig:fidelity-extract} shows the cumulative probability of the JS divergence for each case. For the PPO algorithm, we can observe that the cumulative portion of JS divergence in both cases (1) and (2) increases sharply to 1 (the JS divergence on 97\% states is smaller than 0.05). Since the DQN algorithm is a deterministic method, which always outputs the same optimal action, all the behavior divergence in case (1) is zero and the JS divergence in case (2) also increases sharply to 1. This indicates that the extracted model indeed has very similar behaviors as the targeted model. 

\begin{figure}[t]
\centering
    \begin{minipage}{.49\columnwidth}
        \includegraphics[width=\textwidth]{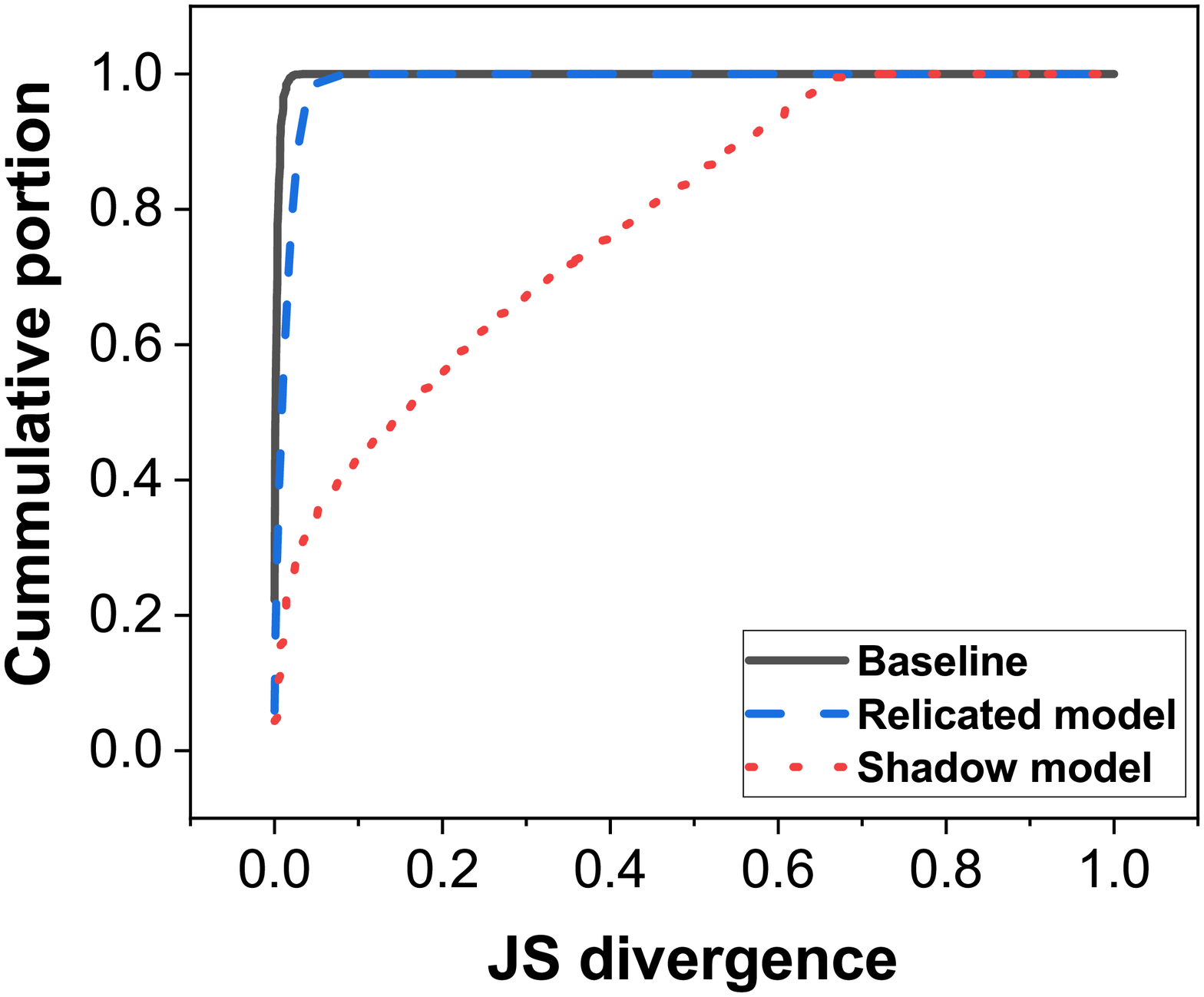}
        \subcaption{Behavior fidelity (PPO)}
        \label{subfig:ppo_model_similarity}
    \end{minipage}
    \begin{minipage}{.49\columnwidth}
        \includegraphics[width=\textwidth]{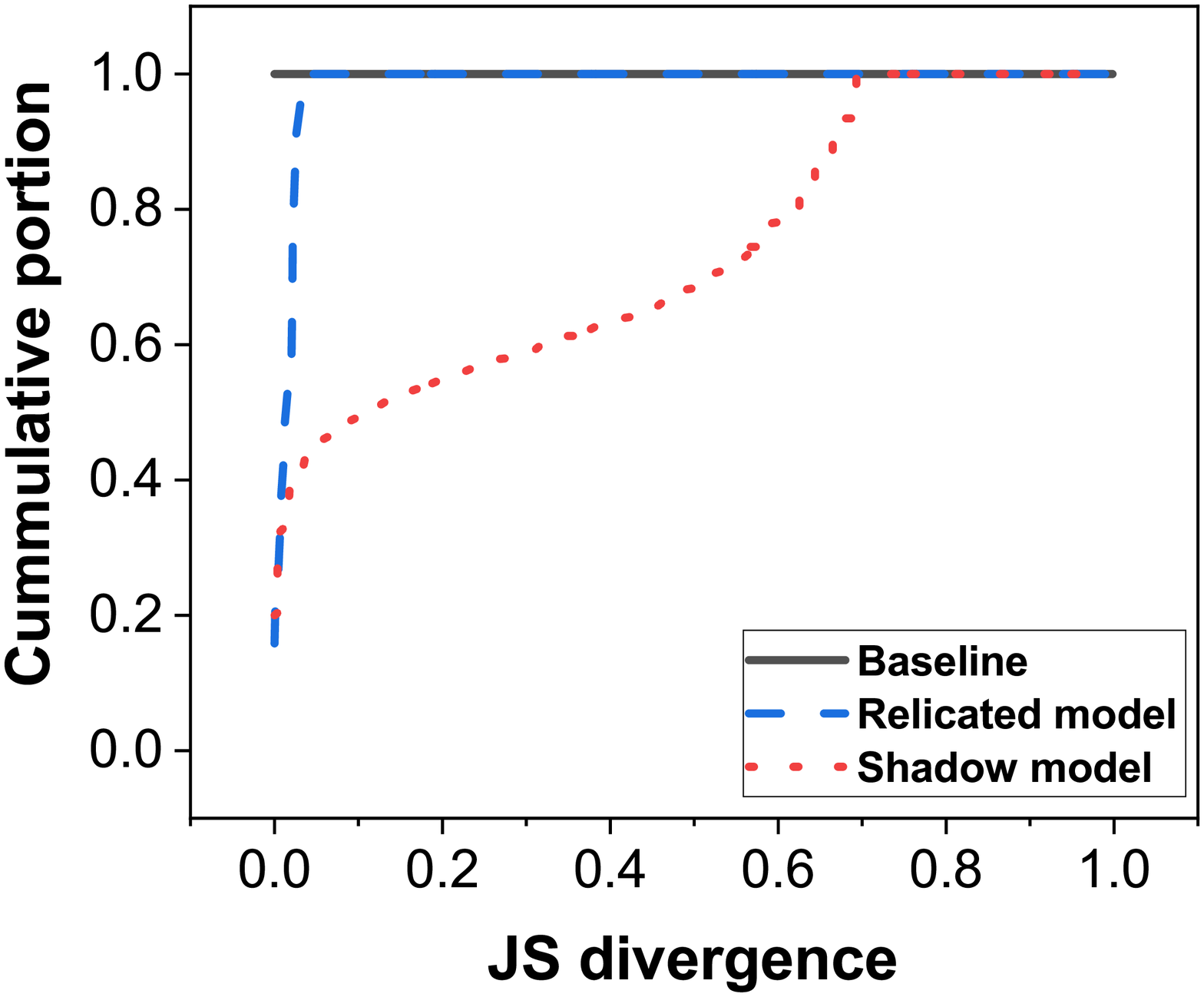}
        \subcaption{Behavior fidelity (DQN)}
        \label{subfig:dqn_model_similarity}
    \end{minipage}
    \caption{Results of fidelity extraction}
    \label{fig:fidelity-extract}
    \vspace{-10pt}
\end{figure}

In contrast, the divergence of action probability distributions between the shadow model and targeted model can be very high. Even they are trained from the same algorithm family, their behaviors are still quite distinct in the same environment. We can conclude that through imitation learning with the identified algorithm family, the extracted model can behave very closely with the targeted one.

\section{Case Study} \label{sec:adv-exp}
To demonstrate the effectiveness and severity of our attack methodology, we present two attack cases. The first case is to utilize model extraction to increase the transferability of adversarial examples against black-box DRL models. The second case is to steal a proprietary model and remove the watermarks embedded in it.

%To verify the effectiveness of our algorithm, we conduct two application cases on DRL. Previous studies have mentioned the transferability of adversarial examples which means that adversarial examples can be transferred among different black-box DRL models. Our first case study shows that transferability of adversarial examples on black-box DRL models can be increased when we have the training algorithm family of DRL models at hand.

%Another crucial security problem about black-box DRL models is privacy, especially in commercial fields. Watermarking technology has been used in neural networks widely to protect NN-based applications. Our second case study shows that the possibility of privacy information leak on black-box DRL models can be increased when we replicate DRL models with proposed method which can remove the embedded watermark.

\subsection{Enhancing Adversarial Attacks}
As the first case, we show how we can leverage model extraction to facilitate existing adversarial attacks against DRL models. 
%In this section, we present a case study to show how an adversary can leverage the model extraction technique to cause severe damages to the victim.

\noindent{\textbf{Motivation.}} 
% support the Adv examples transferability
One of the most severe security threats against machine learning models is the adversarial example \cite{szegedy2013intriguing}: with imperceptible and human unnoticeable modifications to the input, a machine learning model can be fooled to give wrong prediction results. Following this finding, many researchers have designed various methods to attack supervised DL models~\cite{fgsm,carlini2017towards,jsma}. 
% To generate such adversarial examples, several algorithms have been proposed, such as the Fast Gradient Sign Method (FGSM) \cite{fgsm}, and the Jacobian Saliency Map Algorithm (JSMA) approach \cite{jsma}.
Adversarial attacks against RL policies have also received attention over the past years. Huang et al. \cite{huang2017adversarial} made an initial attempt to attack neural network policies by applying the conventional method (FGSM) to the state at each time step. Enhanced adversarial attacks against DRL policies were further demonstrated for higher efficiency and success rates \cite{behzadan2017vulnerability,xiao2019characterizing}.

Adversarial attacks enjoy one important feature: \emph{transferability} \cite{szegedy2013intriguing,xie2019improving}. Malicious samples generated from one model has certain probability to fool other models as well. Due to this transferability, it becomes feasible for an adversary to attack black-box models, as he can generate adversarial examples from an alternative white-box model and then transfer them to the targeted model. 

The transferability of adversarial examples depends on the similarity of the white-box and black-box models. The adversarial examples have higher chance to fool the black-box model which shares similar features and behaviors as the white-box one. This becomes difficult in the scenario of DRL due to the models' complexity and large diversity, especially when they are trained from different algorithms. Hence, we can leverage the model extraction technique proposed in this paper to improve the transferability. Specifically, we identify the training algorithm from the black-box DRL model and replicate a new one. Then we generate adversarial examples via conventional methods from the parameters of the extracted model, and use them to attack the targeted black-box one. 

%Adversarial examples can be transferred between different models. However, the transferability of adversarial examples is not the same among different victim models. The transferability of adversarial examples is higher between two similar models. Therefore, two models trained by the same algorithm have similar transfer feature in DRL. Generally, there are three types of adversarial attacks, white-box, grey-box and black-box attacks, determined by the adversary's knowledge of the victim model~\cite{meng2017magnet}. The black-box scenario is the most realistic setting as the information and details of the DL models are usually confidential for intellectual property protection. However, black-box attacks have lowest success rates due to the low transferability across models with different algorithms. Such distinction is more prominent for DRL models for their complexity and large diversity. So to enhance the adversarial attacks against black-box DRL models, we can use the proposed model extraction attack to turn the black-box models into grey-box ones.

\noindent{\textbf{Implementation.}}
We evaluate the effectiveness of our improved adversarial attacks in Atari Pong. We conduct adversarial attacks on five commonly used DRL algorithm families (i.e., A2C, PPO, ACKTR, ACER, DQN). Each experiment is repeated 3 times and the average results are reported. For each targeted model, we also choose other different DRL algorithms to train shadow models as the baseline. For each white-box shadow model, we generate 1,000 adversarial examples by utilizing the FGSM technique \cite{fgsm}, 
\begin{equation}
    s^{adv} = s + \epsilon sign (\nabla_s L(s, a)),
\end{equation} 
where $\epsilon$ represents the perturbation magnitude set as 0.15 in our experiments. We attack the targeted models with the generated adversarial examples and collect their success rates.

% To estimate the transferability of adversarial examples between different DRL models, we conduct adversarial attacks on the most common DRL algorithms (i.e., A2C, PPO, ACKTR, ACER, DQN). 
% The targeted black-box model adopts one specific training algorithm and configurations. For each type of black-box model, we train three models to be attacked.
% The adversary may choose an arbitrary different algorithm to train a shadow model as the baseline, or use our proposed method to extract a new model. Similarly, the adversary trains three models on each configuration as the white-box model to generate adversarial examples. For each white-box model, the adversary generates 1,000 adversarial examples by utilizing the FGSM technique \cite{fgsm}, 
% \begin{equation}
%     s^{adv} = s + \epsilon sign (\nabla_s L(s, a)),
% \end{equation} 
% where $\epsilon$ represents the perturbation magnitude which is set to 0.15 in our experiments. We attack the black-box models with the generated adversarial examples and measures their success rates on the targeted model.

\noindent{\textbf{Results.}}
Figure \ref{fig:transferability} reports the success rates (transferability) of adversarial examples across different algorithms under the same perturbation scale. We observe that the success rate increases when the extracted model is trained by the same training algorithm family as the targeted model. 

\begin{wrapfigure}{r}{0.6\columnwidth}
    \includegraphics[width=0.6\columnwidth]{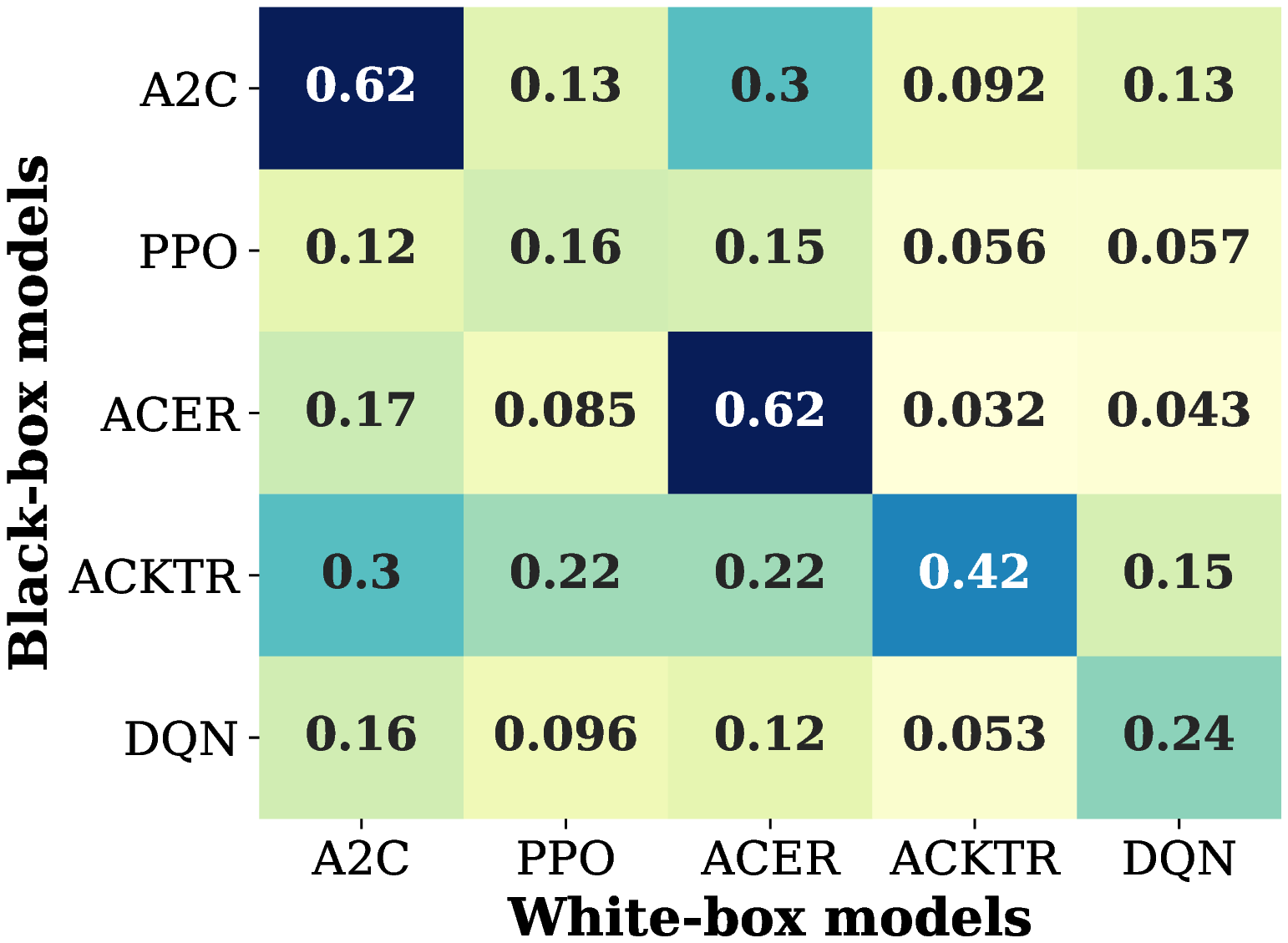}
    \vspace{-20pt}
    \caption{Enhancing transferability}
    \label{fig:transferability}
    \vspace{-5pt}
\end{wrapfigure}
% \noindent{\textbf{Results.}}

The reason is that because the gradients of the DRL models with the same algorithm family are closer than the others, adversarial examples are easier to be transferred. This indicates that our model extraction technique can significantly enhance the black-box adversarial attacks. It is worth noting that since PPO is more robust against adversarial attacks \cite{huang2017adversarial}, the success rates on PPO is generally lower.

\subsection{Invalidating Watermark Protection}
\noindent{\textbf{Motivation.}} 
% Deep neural networks have made tremendous progress in various areas.
DL models have become important Intellectual Property (IP) for AI companies and practitioners. An efficient way for IP protection is watermarking. This technique enables the ownership verification of DL models by embedding watermarks into the models. Then the owners can extract the watermarks from suspicious models as the evidence of ownership. Various techniques have been designed to watermark DNN models \cite{uchida2017embedding,adi2018turning}. In the DRL scenario, the most common approach is to embed a sequential pattern of out-of-distribution states and actions into the targeted DRL model as watermarks \cite{behzadan2019sequential}. In this case study, we show that even a DRL model is protected by this watermark mechanism, our extraction technique is able to replicate the model and remove the embedded watermarks. In another word, our solution is able to learn the behaviors of the targeted model while ignoring the hidden watermark features. 

% , but they also are increasingly targeted by adversaries. Adversaries can steal the model and jeopardize
% the intellectual property of model owners. Digital watermarking is the method to hide the secret information into the digital media in order to protect the ownership of those media data. Recently, \cite{uchida2017embedding} proposed the method for embedding watermark into deep neural networks to verify the ownership of the model. Similarly, to protect the DRL models, \cite{behzadan2019sequential} embeds a sequential pattern of out-of-distribution states and actions into the targeted DRL policy. In this section, we considering to replicate an protected DRL model with the watermarking method and verify the effectiveness of watermarking protection. 

\noindent{\textbf{Implementation.}}
We consider three DRL algorithm families (A2C, PPO, DQN) with the Cart-Pole environment and train the corresponding watermarked DRL models following the implementation in \cite{behzadan2019sequential}.
% , which creates a new state space that dose not belong to the initial state space of the task and induce the DRL agent to master a special transition rules among the out-of-distribution states.  
Specifically, we first create five out-of-distribution watermark. At the $i$-th state, we define the corresponding action and the next state as $i\%2$ and state[$i\%4+1$], respectively. To embed the watermarks, we reset the reward of the predefined action on the $i$-th state to 1. If the DRL model acts differently from the watermark actions, we compulsorily terminate the episode and return a reward of -1. During the DRL training and watermark embedding process, we train DRL models in the new verification environment until the watermarks are successfully embedded and the watermarked models can achieve high reward values (i.e., 195). For each DRL algorithm family, we train three watermarked models with the above scheme and report the average results.

\begin{figure}[t]
    \centering
    \includegraphics[width=\columnwidth]{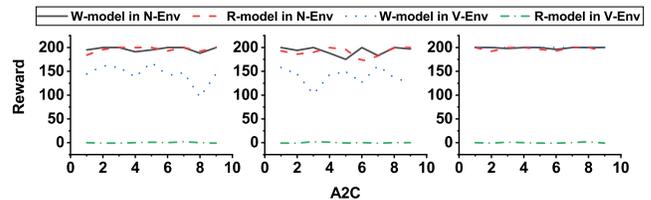}
    \caption{The performance of watermarked and extracted models in the normal and verification environments}
    \label{fig:watermark}
    \vspace{-5pt}
\end{figure}

\noindent{\textbf{Results.}}
Figure \ref{fig:watermark} shows the performance of the targeted watermarked model, as well as the extracted model using our proposed technique. The $x$-axis and $y$-axis represent the index of the experiment and the corresponding reward, respectively. In each figure, we show four lines: the watermarked model in the normal environment (W-model in N-Env) and verification environment (W-model in V-Env); the replicated model in the normal environment (R-model in N-Env) and verification environment (R-model in V-Env). 

We observe that, in the normal environment, the performance of the watermarked model and the replicated model is identical, indicating the success of model extraction. In contrast, in the verification environment, the replicated model has quite different behaviors and reward from the watermarked model. We conclude that the model extraction attack will not learn the watermarks from the targeted model, for all the three algorithms. We also observe that the watermarked model in the case of DQN outperforms other algorithms in the verification environment. This is because the DQN algorithm is deterministic and always outputs the optimal actions, while PPO and A2C are stochastic that sample actions following their action probabilities. Thus, if a watermarked model samples an action different from the predefined action, PPO and A2C will terminate the process and get a lower reward.

\section{Discussion}\label{sec:discussion}

\subsection{Attack Complexity}
The cost of our attack originates from algorithm classification and imitation learning. For the first step, the main effort is to train a number of DRL models for data collection, and train an algorithm classifier. Note that this offline step incurs only one-time cost. After the classifier is produced, we can use it to attack arbitrary DRL models for infinite times. This makes the training effort acceptable. Future works will focus on reducing the number of shadow models to maintain or improve the classification performance. The online prediction process is very fast with negligible overhead.

For the second step, the adversary needs to learn a model via imitation learning. This is comparable to training a new model from scratch. Sometimes the adversary may not obtain an ideal model with the same reward due to the high stochasticity in the DRL learning process. He has to repeat this step until a qualified model is identified. From our empirical results, the adversary can find the optimal model in a couple of rounds with very high probability if he uses the correct training algorithm. Besides, the adversary can perform fine-tuning over a well-trained model (e.g., a shadow model from the same training family at step 1) to speed up this imitation learning process.

%Parameter settings would affect the accuracy and efficiency of the proposed model extraction attack. As the number of queries increases, the adversary is more likely to train a satisfactory RNN classifier and learn a competent model via imitation learning. However, the computational complexity would increase as the training samples enlarge. Stronger RNN classifiers and imitation learning algorithms are more welcome but LSTM and GAIL are powerful enough for DRL model extraction in our experiments. In our paper, we assume that the targeted model belongs to one of the algorithm families in the pool, which cannot be always true. To increase the success probability, the adversary can collect more algorithm families for RNN classification but, of course, would expend more computing resources.

\subsection{Extension to More Complex Scenarios}
Our attack method follows the general modeling of DRL scenarios to predict and extract the targeted model via the observed actions. In our experiments, we evaluate the proposed attack on two standard DRL benchmark tasks (i.e., Cart-Pole and Atari Pong). It is expected to also work on other more complex DRL applications such as robotics navigation \cite{zhu2017target} and motion control \cite{lillicrap2015continuous}. 

Our approach can be extended to other algorithms as well. We can train a classifier to recognize the features of new algorithms in a similar way. Particularly, some reinforcement learning algorithms may have states and actions that are not fully observable (e.g., partially observable Markov Decision Process \cite{cassandra1998survey}). In this case, the adversary can just select the observable state-action pairs while ignoring the hidden ones. Then he can follow our attack technique to perform model extraction. 

Another complex case is that the training algorithm of the targeted model may not fall into the candidate pool. For instance, the model can be an ensemble of multiple DRL and CNN models; the task is modeled as a multi-agent reinforcement learning problem. In this case, our classifier is not able to predict the algorithms correctly. To the best of our knowledge, extracting a combination of multiple models is an open problem and has not been solved yet. This will be an promising future direction in our consideration. 

%In the paper, we assume that the targeted model belongs to one of the algorithm families in the pool of the adversary. We have shown that our attack are robustness without being affected by network structures, loss functions, and other hyperparameters. However, our attack may be ineffective once the targeted modal is much different from the the algorithm families in the pool. For example, the targeted model is a combination of multiple DRL and CNN models. Model extraction under this scenario is an open problem and quite challenging. We plan to address this problem in the future by first speculating the structure of the targeted model and then implementing the extraction process.

%The reasons are two-folds. First, the propose attack requires very few information (i.e., states and actions) and thus it can be easily extended to all DRL-based applications where the states can be captured and the corresponding actions of the targeted model can be observed. Second, in the case where the states and actions can not be fully captured or observed , adversaries can still selectively collect states and actions. Then, they can follow our attack, i.e., adopting the RNN classifier to identify the training family and then replicating victim models using sequential imitation learning.

\subsection{Potential Defenses}
As discussed in Section \ref{sec:past-defense}, existing defenses mainly focus on extraction of supervised DL models. It is hard to apply them to protect DRL models. Hence, a defense solution specifically designed for our attack is required. We propose three possible ideas. Validation and implementation of these solutions will be our future work. 

The first possible defense is to distinguish the query behaviors of imitation learning from normal operations. Imitation learning may exhibit specific query patterns, which can indicate the occurrence of model extraction. The second defense is to add higher stochasticity to the model's behaviors, making it harder for the adversary to recognize the algorithm or mimic exact behaviors. There can be a trade-off between the model usability and security that needs to be balanced. The third defense idea is to design new DRL training algorithms totally different from the ones in the family pool. If these algorithms are kept secret, the adversary is not able to perform accurate imitation learning.

\section{Related Work}\label{sec:work}
% \textbf{DNN Model Privacy.} 

%A quantity of works have been devoted to study privacy threats in deep learning which can be classified into two categories: model inversion and extraction attacks. While model inversion attacks \cite{fredrikson2015model} aim to inverse training data properties such as membership, classification representatives, and training data, model extraction attacks intend to steal model parameters or architectures of the targeted model \cite{tramer2016stealing, oh2019towards}. In this paper, we focus on model extraction threats in DRL systems. 

% Membership inference attacks \cite{shokri2017membership} are designed to determine if a given data sample has been included in the training data.
%With the information extracted from the targeted model, an adversary can reconstruct a substitute model with similar behavior as the targeted model and hence violating the intellectual property of the service provider.

\subsection{DNN Model Extraction Attacks}
Model extraction attacks aim to steal a DNN model (i.e., acquiring the attributes or parameters of the model, or a good approximation) with only black-box access. These attacks can be classified into two categories. The first category is query-based attacks: the adversary treats the targeted model as an oracle, and queries it with carefully-crafted samples. The model is replicated based on the corresponding responses. Tramer et al. \cite{tramer2016stealing} realized this attack against modern machine learning services. Wang and Gong \cite{wang2018stealing} adopted similar techniques to steal the hyper-parameters in model training. Oh et al. \cite{oh2019towards} extracted the network structure and configurations of the targeted model (e.g., number of layers, optimization algorithm, and activation function). 

Advanced attacks were proposed to improve the efficiency of model extraction with fewer queries, like adversarial example-based query \cite{yu2020cloudleak}, learning-based query \cite{jagielski2020high}, gradients query \cite{milli2019model}. Model extraction attacks leveraging different types of query data were also demonstrated in \cite{orekondy2019knockoff,pal2019framework,correia2018copycat}. Chandrasekaran et al. \cite{chChGi:18} showed the process of model extraction is similar to active learning, which can further enhance the attack efficiency. Carlini et al. \cite{carlini2020cryptanalytic} treated the model extraction as a cryptanalytic problem, and applied the differential attack technique to solve this issue.

%introduced a query-based model extraction attack against DNNs. The attack allows the adversary queries the targeted models and collects the corresponding predictions. The adversary can successfully steal the model by training a substitute model with these input-output pairs. Following the design strategy in \cite{tramer2016stealing}, several model extraction attacks were proposed to increase the accuracy of the extracted models. For example, Wang et al. \cite{wang2018stealing} provide the output distribution of all labels and extract models that are equivalent to the targeted ones. Oh et al. \cite{oh2019towards} also utilize the query-based mechanism to extract the confidential attributes of the targeted model (e.g., number of layers, optimization algorithm, and activation function). 

The second category is side-channel-based attacks. The adversary collects side-channel information emitted during the inference phase to infer the properties of the neural network models. This includes timing channels \cite{duddu2018stealing}, power side channels \cite{batina2019csi}, and micro-architectural side channels \cite{hong2018security,hua2018reverse}.
%Several model extraction attacks \cite{chChGi:18,jagielski2020high} leverage existing learning techniques such active learning and semi-supervised learning to improve model extraction attacks. For example, Chandrasekaran et al. \cite{chChGi:18} shows the process of model extraction is similar to active learning which is then used to extract machine leanring model such as SVM, decision trees, and random forests. 

All of the above extraction works focus on conventional neural networks. In contrast, this paper presents the first attack against DRL models. Due to the distinct features and complex mechanisms of DRL models, it is hard to directly apply the prior techniques to address this problem. We propose a novel attack method integrating algorithm prediction and imitation learning to achieve DRL model extraction with high fidelity and efficiency.

\subsection{Defenses against Model Extraction}
\label{sec:past-defense}
Past works also introduced different defenses to thwart model extraction threats, which can be classified into three categories. The first type of solution is prediction modification. Inspired by the information theory, these solutions limit the information gained per query from the adversary to increase the difficulty of model extraction, including perturbing the output probability \cite{tramer2016stealing,chChGi:18}, removing the probabilities for some classes \cite{tramer2016stealing}, returning only the class output \cite{tramer2016stealing,chChGi:18}. The second type is to detect extraction adversaries from benign users based on query pattern analysis \cite{juuti2019prada,kesarwani2018model}. The third type of protection is to embed watermarks into the targeted models, and identify the extracted model copy by verifying the ownership \cite{zhang2018protecting,uchida2017embedding}.

%We can classify existing defenses into three categories: prediction modification, query pattern analysis \cite{juuti2019prada,kesarwani2018model}, and watermarking \cite{zhang2018protecting,uchida2017embedding}. On the bases of information theory, the defenses of the first category limit the information gained per query from the adversary to The second approach attempts to differentiate extraction adversaries from benign users and thus achieves defending model extraction attacks. Instead of directly preventing models from stealing, watermarking-based defenses embed watermarks into the targeted models. The extracted models are still watermarked and can be verified. 

Unfortunately, these solutions cannot defeat our DRL model extraction attacks. The adversary just uses normal states to query the targeted model, and observes the output actions instead of the confidence scores. As such, it is hard to prevent such malicious behaviors via prediction modification or query pattern analysis. We also show that our extracted model does not contain the watermarks anymore, thus invalidating the digital watermarking protection. 

%Challenges rise when applying existing state-of-the-art defenses for DNN extraction attacks to DRL scenarios. One the one hand, some of the defenses cannot apply to defend our attack or have been proven to be ineffective (e.g. Watermarking \cite{quiring2017fraternal}, Mlcapsule \cite{hanzlik2018mlcapsule}). On the other hand, although the prediction modification defenses \cite{lee2018defending} can be applied to DRL scenarios, it falls to defend our attack since the DRL environments may only output the optimal actions (i.e., Top-1 predictions) instead of the whole prediction distributions.

\section{Conclusion}\label{sec:conclusion}

In this paper, we design a novel attack methodology to steal DRL models. We draw the insight that DRL model extraction can be analyzed as an imitation learning process. Hence, we propose to leverage state-of-the-art imitation learning techniques (e.g., Generative Adversarial Imitation Learning) to perform the attacks. To improve the extraction fidelity, we propose to build a classifier to predict the training algorithm family of the targeted model with only black-box access. The integration of algorithm prediction and imitation learning can achieve DRL model extraction with high accuracy and fidelity. Our attack technique can be used to enhance adversarial attacks and invalidate watermarking mechanisms. We expect this study can inspire people's awareness about the severity of DRL model privacy issue, and come up with better solutions to mitigate such model extraction attack.

\bibliographystyle{ACM-Reference-Format}
\bibliography{references}

\end{document}